\documentclass[twoside,11pt]{article}
\usepackage{jair, theapa, rawfonts}

\jairheading{77}{2023}{71-101}{01/2023}{05/2023}
\ShortHeadings{Exploiting AIR and Exogenous State Variables for Offline RL}
{Liu, Wright, \& White}
\firstpageno{71}

\usepackage{amsmath,amsthm,amsfonts,bbm}
\usepackage{thmtools,thm-restate}
\usepackage{graphicx,subfigure}
\usepackage{color,amssymb}
\usepackage{url}
\usepackage{algorithm}
\usepackage[noend]{algorithmic}
\theoremstyle{definition}

\newcommand{\norm}[1]{\| #1 \|}

\newcommand{\bbP}{\mathbb{P}}
\newcommand{\cF}{\mathcal{F}}

\newcommand{\RR}{{\mathbb{R}}}
\newcommand{\EE}{{\mathbb{E}}}

\newcommand{\ZZ}{{\mathbb{Z}}}
\newcommand{\II}{{\mathbb{I}}}
\newcommand{\VV}{{\mathbb{V}}}

\newcommand{\calU}{{\mathcal{U}}}

\newtheorem{lemma}{Lemma}

\newtheorem{definition}{Definition}
\newtheorem{assumption}{Assumption}

 \newcommand{\defeq}{:=}

\newcommand{\bellman}[1]{\mathcal{T}#1}
\newcommand{\hatbellman}[1]{\mathcal{\hat T}#1}

\newcommand{\mynorm}[1]{\left\lVert#1\right\rVert}

\newcommand{\cS}{\mathcal{S}}
\newcommand{\cA}{\mathcal{A}}

\newcommand{\calF}{\mathcal{F}}

\definecolor{darkred}{rgb}{0.7, 0.0, 0.0}
\definecolor{darkgreen}{rgb}{0.0, 0.7, 0.0}

\begin{document}

\title{Exploiting Action Impact Regularity and Exogenous State Variables for Offline Reinforcement Learning}

\author{\name Vincent Liu \email vliu1@ualberta.ca \\
       \name James R. Wright \email james.wright@ualberta.ca \\
       \name Martha White \email whitem@ualberta.ca \\
       \addr University of Alberta and Alberta Machine Intelligence Institute (Amii)\\
       Edmonton, Alberta, Canada
}

\maketitle

\begin{abstract}
Offline reinforcement learning---learning a policy from a batch of data---is known to be hard for general MDPs. These results motivate the need to look at specific classes of MDPs where offline reinforcement learning might be feasible. In this work, we explore a restricted class of MDPs to obtain guarantees for offline reinforcement learning. The key property, which we call Action Impact Regularity (AIR), is that actions primarily impact a part of the state (an endogenous component) and have limited impact on the remaining part of the state (an exogenous component). AIR is a strong assumption, but it nonetheless holds in a number of real-world domains including financial markets. We discuss algorithms that exploit the AIR property, and provide a theoretical analysis for an algorithm based on Fitted-Q Iteration. Finally, we demonstrate that the algorithm outperforms existing offline reinforcement learning algorithms across different data collection policies in simulated and real world environments where the regularity holds.
\end{abstract}

\section{Introduction}
Offline reinforcement learning (RL) involves using a previously collected static dataset, without any online interaction, to learn an output policy. This problem setting is important for a variety of real world problems where learning online can be dangerous, such as for self-driving cars, or when building a good simulator is difficult or costly, such as for healthcare. It is also a useful setting for applications where there is a large amount of data available, such as dynamic search advertising.

A challenge in offline RL is that the quality of the output policy can be highly dependent on the data.
Most obviously, the data might not cover some parts of the environment, resulting in two issues.
The first is that the learned policy, when executed in the environment, is likely to deviate from the behavior that generated its training data and reach a state-action pair that was unseen in the dataset. For these unseen state-action pairs, the algorithm has no information about how to choose a good action. The second issue is that 
if the dataset does not contain transitions in the high-reward regions of the state-action space, it may be impossible for any algorithm to return a good policy. One can easily construct a family of MDPs with missing data such that no algorithm can identify the MDP and suffer a large suboptimality gap \cite{chen2019information}.  

The works that provide guarantees on the suboptimality of the output policy usually rely on strong assumptions about good data coverage and mild distribution shift. The theoretical results are for methods based on approximate value iteration (AVI) and approximate policy iteration (API), with results showing the output policy is close to an optimal policy \cite{farahmand2010error,munos2003error,munos2005error,munos2007performance}.
They assume a small concentration coefficient, which is the ratio between the state-action distribution induced by any policy and the data distribution \cite{munos2003error}. 
However, the concentration coefficient can be very large or even infinite in practice, so assuming a small concentration coefficient can be unrealistic for many real world problems. 
Different measures of distribution shift are also used, for example, \citeA{yin2021near} assume the visitation distribution of the least occupied state-action pair is greater than zero.

To avoid making strong assumptions on the concentration coefficient, several works consider constraining divergence between the behavior policy and the output policy on the policy improvement step for API algorithms. The constraints can be enforced either as direct policy constraints or by a penalty added to the value function~\cite{levine2020offline,wu2019behavior}. 
Another approach is to constrain the policy set such that it only chooses actions or state-action pairs with sufficient data coverage when applying updates for AVI algorithms \cite{kumar2019stabilizing,liu2020provably}. However, these methods only work when the data collection policy covers an optimal policy (see our discussion in Section \ref{sec:ass}), which can be difficult or impossible to guarantee. 

Another direction has been to assume pessimistic values for unknown state-action pairs, to encourage the agent to learn an improved policy that stays within the parts of the space covered by the data.
CQL \cite{kumar2020conservative} penalizes the values for out-of-distribution actions and learns a lower bound of the value estimates. 
A related idea is to constrain the bootstrap target to avoid out-of-distribution actions, introduced first in the BCQ algorithm \cite{fujimoto2019off} with practical improvements given by IQL \cite{kostrikov2021offline}.
MOReL \cite{kidambi2020morel} learns a model and an unknown state-action detector to partition states similar to the R-Max algorithm \cite{brafman2002r}, but then uses the principle of pessimism for these unknown states rather than optimism.
Safe policy improvement methods~\cite{thomas2015high,laroche2019safe} rely on a high-confidence lower bound on the output policy performance, performing policy improvement only when the performance is higher than a threshold. 

In practice, the results of pessimistic approaches are mixed. Some have been shown to be effective on the D4RL dataset \cite{fu2020d4rl}. 
Other results, however, show methods can be too conservative and fail drastically when the behavior policy is not near-optimal \cite{kumar2020conservative,liu2020provably}, as we reaffirm in our experiments. 
Further, actually using these pessimistic methods can be difficult, as their hyper-parameters are not easy to tune in the offline setting \cite{wu2019behavior} and some methods require an estimate of the behavior policy or data distribution \cite{kumar2019stabilizing,liu2020provably}.  

\begin{figure}
    \centering
\includegraphics[width=0.8\linewidth]{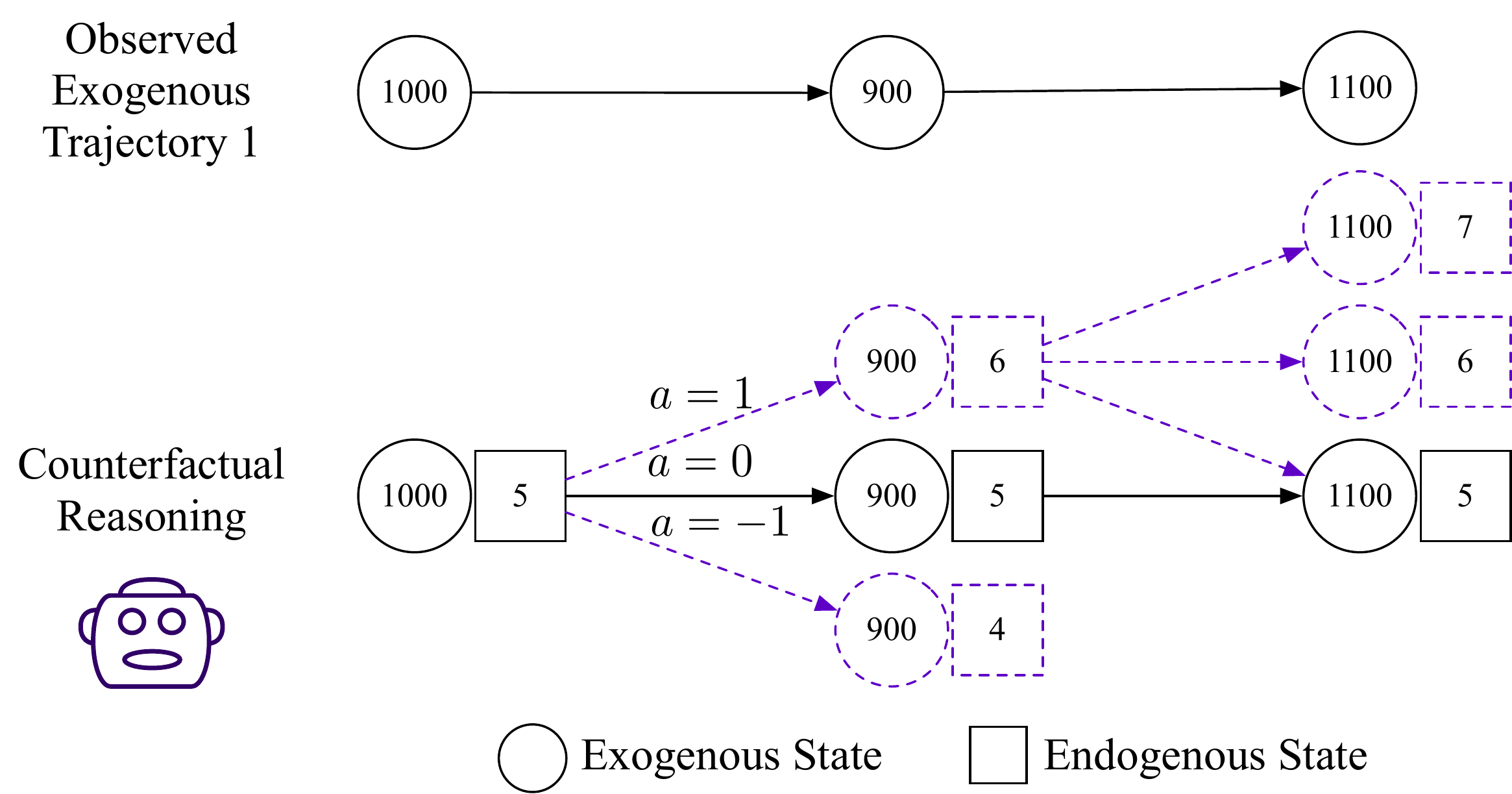}
    \vspace{-0.4cm}
    \caption{In the stock market example, the exogenous state corresponds to the stock price, the endogenous state corresponds to number of shares the agent has, and the action corresponds to the number of share to buy or sell at each time step. Given an observed exogenous trajectory, the agent can counter-factually reason about the outcomes of different actions and endogenous state.}
    \label{fig:exo_diagram}
        \vspace{-0.4cm}
\end{figure}

Intuitively, however, there are settings where offline RL should be effective. Consider a trading agent in a stock market. A policy that merely observes stock prices and volumes without buying or selling any shares provides useful information about the environment. For this collected dataset, an offline agent can counter-factually reason about the utility of many different actions as demonstrated in Figure \ref{fig:exo_diagram}, because its actions have limited impact on the prices and volumes. 
Such MDPs, which are called Exogenous MDPs, have states that separate into exogenous states (stock price), not impacted by actions, and endogenous states (number of shares owned by the agent). 
The structure in Exogenous MDPs has been used in online RL to learn more efficiently \cite{dietterich2018discovering}.

This exogenous structure, however, has yet to be formally investigated for offline RL, though it is likely already being exploited in industry. Exploiting this structure is natural in applied financial applications, because datasets allow for alternative trajectories to be simulated, as described in the above example. One (unpublished) system uses RL and trajectory simulation for the optimal order execution problem \cite{aiden2020}; it seems likely that there are other such systems in use. What has yet to be done, however, is to understand the theoretical properties of such algorithms, as well as potential algorithmic improvements.

In this paper, we first generalize the definition of exogenous MDPs, and formalize the action impact regularity (AIR) property. We say an MDP has the AIR property---or is $\varepsilon$-AIR---if the actions have a limited impact on exogenous dynamics, with the level of impact determined by $\varepsilon \ge 0$. This generalizes the previous definition, which required strict separation, corresponding to $\varepsilon = 0$. We develop both theory and algorithms for this more general setting, assuming access only to an offline dataset, an approximate (learned) endogenous model and the reward function.\footnote{We could assume an approximate instead of exact reward model. However, in most RL analysis, the error on the transition model is of a higher order than the error for reward models (for example, see \citeA{agarwal2020model}). For simplicity, it is often assumed the reward model is known.} 

We design an efficient algorithm, called FQI-AIR, to exploit the AIR property, that (1) does not require an estimate of the behavior policy or the data distribution, (2) has a straightforward approach to select hyperparameters using just the given data and (3) is much less sensitive to the quality of the data collection policy, if our assumptions hold. This algorithm is a simple extension of FQI, but is significantly more computationally efficient than the trajectory simulation approach mentioned above and allows us to leverage and extend the existing theory for FQI. We bound the suboptimality of the output policy from FQI-AIR, in terms of $\varepsilon$ and other standard terms (model errors and the inherent Bellman error, see Section \ref{sec_main_theory}). Importantly, in place of the concentration coefficient, we have a term that depends on the size of the endogenous state and number of actions; when the concentration coefficient is bounded and small, it is on the same order as this term. 

We then conduct a comprehensive empirical study of FQI-AIR. We compare several algorithms in two simulated environments, across three different data collection policies, with varying offline dataset sizes, for $\varepsilon = 0$ (assumption perfectly satisfied) and a larger $\varepsilon$ (assumption somewhat violated). FQI-AIR significantly outperforms the offline RL algorithms that do not leverage the AIR property---including FQI, MBS-QI, CQL and IQL; this outcome is expected, but nonetheless verifies that exploiting the AIR property, when appropriate, can have a big benefit. We show that these conclusions extend to two environments based on real-world datasets (for bitcoin trading and for controlling battery usage in a hybrid car). An important detail here is how hyperparameters are chosen. FQI-AIR can exploit AIR for policy evaluation, to automatically select hyperparameters. For the other algorithms, we do not have such an approach, and instead report idealized performance by picking hyperparameters based on performance in the environment. 

Finally, these results all used the true endogenous model for FQI-AIR. We chose to do so partly because the endogenous model is known for certain AIR-MDPs (e.g., trading, inventory management) and partly to focus the investigation on the role of $\varepsilon$ rather than model error. However, for certain AIR-MDPs, we will not have access to the true endogenous model. In our final experiment, we investigated the impact of using a learned endogenous model in the hybrid car environment, and show that FQI-AIR remains effective.

All of this is only possible because we make a strong assumption about the environment. However, given the hardness results in offline RL, we should acknowledge that we likely need to restrict the class of MDPs. This work is a step towards understanding for what classes of MDPs offline RL is feasible. At the same time, though we consider a restricted setting, it is by no means a trivial setting. There are many real-world examples where this regularity holds (as we discuss later in this work). This is doubly true given that our generalization provides some flexibility in violating the assumption: the regularity only needs to hold approximately rather than exactly. The algorithms and theory developed here can benefit these real-world applications now, by providing an approach that is well-designed and well-behaved with strong theoretical guarantees for their specific problem setting.

\section{Problem Formulation}
\newcommand{\rmax}{{r_{max}}}
\newcommand{\vmax}{{v_{max}}}
\newcommand{\tenv}{\text{exo}}
\newcommand{\tpri}{\text{end}}
\newcommand{\env}[1]{#1^\tenv}
\newcommand{\private}[1]{#1^\tpri}
\newcommand{\Senv}{\mathcal{S}^{\tenv}}
\newcommand{\Spri}{\mathcal{S}^{\tpri}}
\newcommand{\senv}{\env{s}}
\newcommand{\spri}{\private{s}}
\newcommand{\spenv}{s'^\tenv}
\newcommand{\sppri}{s'^\tpri}
\newcommand{\Penv}{P^\tenv}
\newcommand{\Ppri}{P^\tpri}
\newcommand{\TPenv}{\Tilde P^\tenv}
\newcommand{\TPpri}{\Tilde P^\tpri}
\newcommand{\HPenv}{\Hat P^\tenv}
\newcommand{\HPpri}{\Hat P^\tpri}
\newcommand{\varapx}{\varepsilon_{apx}}
\newcommand{\varair}{\varepsilon_{air}}
\newcommand{\varp}{\varepsilon_{p}}

The agent-environment interaction is
formalized as a finite horizon Markov decision process (MDP) $M=(\cS, \cA, P, r, H, \nu)$. $\cS$ is a set of states, and $\cA$ is an set of actions; for simplicity, we assume that both sets are finite. $P:\cS\times\cA\to\Delta(\cS)$ is the transition probability where $\Delta(\cS)$ is the set of probability distributions on $\cS$, and by slightly abusing notation, we will write $P(s,a,s')$ as the probability that the process will transition into state $s'$ when in state $s$ it takes action $a$. The function $r:\cS\times\cA\to[0, \rmax]$ gives the reward when taking action $a$ in state $s$, where $\rmax \in \mathbb{R}^+$. $H\in\ZZ^+$ is the planning horizon, and $\nu\in\Delta(\cS)$ the initial state distribution. 

In the finite horizon setting, the policies are non-stationary. A non-stationary policy is a sequence of memoryless policies $(\pi_{0},\dots,\pi_{H-1})$ where $\pi_h:\cS\to\Delta(\cA)$. We assume that the set of states reachable at time step $h$, $\cS_h\subset\cS$, are disjoint, without loss of generality, because we can always define a new state space $\cS' = \cS \times [H-1]$ where $[n]\defeq\{0,1,2,\dots,n\}$. Then, it is sufficient to consider stationary policies $\pi:\cS\to\Delta(\cA)$. 

Given a policy $\pi$, $h\in[H-1]$, and $(s,a)\in\cS\times\cA$, we define the value function and the action-value function as $v_h^\pi(s) \defeq \EE^\pi[\sum_{t=h}^{H-1} r(S_t,A_t)|S_h\!=\!s]$ and $q_h^\pi(s,a) \defeq \EE^\pi[\sum_{t=h}^{H-1} r(S_t,A_t)|S_h\!=\!s,A_h\!=\!a]$ where the expectation is with respect to $\bbP^\pi_M$ (we may drop the subscript $M$ when it is clear from the context). $\bbP^\pi$ is the probability measure on the random element in $(\cS\times\cA)^H$ induced by the policy $\pi$ and the MDP such that, for the trajectory of state-action pairs $(S_0,A_0,\dots,S_{H-1},A_{H-1})$, we have $\bbP^\pi(S_0=s)=\nu(s)$,  $\bbP^\pi(A_t\!=\!a|S_0,A_0,\dots,S_{t})\!=\!\pi(a|S_t)$, and $\bbP^\pi(S_{t+1}\!=\!s'|S_0,A_0,\dots,S_t,A_t)=P(S_t,A_t,s')$ for $t\geq0$ \cite{lattimore2020bandit}. 
The optimal value function is defined by $v_h^*(s) \defeq \sup_{\pi} v_h^\pi(s)$, and the Bellman operator is defined by
\begin{align*}
    (\bellman q_{h})(s,a) = r(s,a) + \sum_{s'\in\cS} P(s,a,s') \max_{a'\in\cA} q_{h}(s',a').
\end{align*}
In the batch setting, we are given a fixed set of transitions $D$ with samples drawn from a data distribution $\mu$ over $(\cS\times\cA)$. In this paper, we consider the setting where the data is collected by a data collection policy $\pi_b$. That is, $D$ consists of $N$ trajectories $(S_0^{(i)},A_0^{(i)},\dots,S_{H-1}^{(i)},A_{H-1}^{(i)})$ induced by the the interaction of the policy $\pi_b$ and MDP $M$.

A representative algorithm for the batch setting is \emph{Fitted Q Iteration} (FQI) \cite{ernst2005tree}. In the finite horizon setting, FQI learns an action-value function for each time step, $q_0,\dots,q_{H-1}\in\cF$ where $\cF\subseteq\RR^{\cS\times\cA}$ is the value function class. The algorithm is defined recursively from the end of the episode: for each time step $h$ from $H-1$ to $0$, $q_h = \arg\min_{q\in\cF} \norm{q - \hatbellman q_{h+1}}_2^2$ where $\hatbellman$ is defined by replacing expectations with sample averages for the Bellman operator $\bellman$ and $q_{H}=0$.
The output policy is obtained by greedifying according to these action-value functions.

\section{Common Assumptions for Offline RL}\label{sec:ass}
In this section we discuss common assumptions used in offline RL.
The most common assumptions typically concern properties of the data distribution and MDPs together: either that it sufficiently covers the set of possible transitions, or that it sufficiently covers a near-optimal policy.
However, these assumptions are often impractical for many real world applications.
Furthermore, recent results show that restriction on the data distribution alone are insufficient to obtain guarantees.
This motivates considering realistic assumptions on the MDP alone, as we do in this work.

Sufficient coverage of the data distribution has primarily been quantified by \emph{concentration coefficients} \cite{munos2003error}. Given a data distribution $\mu$, the concentration coefficient $C$ is defined to be the smallest value such that, for any policy $\pi$, 
\begin{equation*}
    \max_{h\in[H-1]}\max_{s\in\cS_h,a\in\cA_h} \frac{\bbP^\pi(S_h=s,A_h=a)}{\mu(s,a)} \leq C.
\end{equation*} 
If $\mu(s,a)=0$ for some $(s,a)$, then we define $C=\infty$ by convention. 
Several results bound the suboptimality of batch API and AVI algorithms in terms of the concentration coefficient~\cite{chen2019information,farahmand2010error,munos2003error,munos2007performance}. For example, it has been shown that FQI outputs a near-optimal policy when $C$ is small and the value function class is rich enough, where the upper bound on the suboptimality of the output policy scales linearly with $\sqrt{C}$. 

However, in practice, the concentration coefficient can be very large or even infinite. For example, if the data collection policy is not well-randomized or exploratory---often the case in practice---then the concentration coefficient is infinite due to missing some state-action pairs.  \citeA{munos2007performance} provides some intuition about the size of the concentration coefficient. 
Suppose that the data distribution is uniform (e.g., $\mu(s,a)=1/|\cS||\cA|$) and the environment transition probability is less uniform, that is, there exists some policies such that the visitation distribution concentrates on a single state-action pair (e.g., $\bbP^\pi(S_h=s,A_h=a)=1$ for some $s,a$ and $h$), then the concentration coefficient can be as large as the number of state-action pairs

Another direction has been to consider approaches where the data covers a near-optimal policy. The key idea behind these methods is to restrict the policy to choose actions that have sufficient data coverage, which is effective if the given data has near-optimal action selection. For example, BCQ \cite{fujimoto2019off} and BEAR \cite{kumar2019stabilizing} only bootstrap values from actions $a$ if the probability $\pi_b(a|s)$ is above a threshold $b$. 
MBS-QI \cite{liu2020provably} extends this to consider state-action probabilities, only bootstrapping from state-action pairs $(s,a)$ when $\mu(s,a)$ is above a threshold. 
The algorithm is modified from FQI by replacing the bootstrap value $q_h(s,a)$ by $\Tilde q_h(s,a) \defeq \II\{\mu(s,a) \geq b\} q_h(s,a)$ and the policy is greedy with respect to $\Tilde q_h(s,a)$. 
If a state-action pair does not have sufficient data coverage, its value is zero. They show that MBS-QI outputs a near-optimal policy if $\bbP^{\pi^*} (\mu(S_h,A_h)\! <\! b)$ is small for all $h\!\in\![H\!-\!1]$. That is, the data provides sufficient coverage for state-action pairs visited under an optimal policy $\pi^*$. 

Though potentially less stringent than having a small concentration coefficient, this assumption can be impractical. We may be able to satisfy this assumption in simulated environments, such as those in our experiments; in the real world, though, if we have a simulator we are unlikely to use offline RL. 
For many real world domains, optimal policies are unknown. In fact, one of the primary purposes of using offline RL is to get (significantly) improved policies. It is also hard to carefully design a data collection policy to cover an unknown optimal policy, making it difficult to even check whether this assumption holds. 

Finally, some recent negative results suggest it is not sufficient to have a good data distribution alone, and that it will be necessary to also make assumptions on the MDP. In particular, 
\citeA{chen2019information} showed that if we do not make assumptions on the MDP dynamics, no algorithm can achieve a polynomial sample complexity to return a near-optimal policy, even when the algorithm can choose any data distribution $\mu$. 
\citeA{wang2020statistical} provide an exponential lower bound for the sample complexity of off-policy policy evaluation and optimization algorithms with $q^\pi$-realizable linear function class, even when the data distribution induces a well-conditioned covariance matrix.
\citeA{zanette2020exponential} provide an example where offline RL is exponentially harder than online RL, even with the best data distribution, $q^\pi$-realizable function class and assuming the exact feedback is observed for each sample in the dataset. 
\citeA{xiao2021sample} provide an exponential lower bound for the sample complexity of obtaining nearly-optimal policies when the data is obtained by following a data collection policy.
These results are consistent with the above, since achieving a small concentration coefficient implicitly places assumptions on the MDP. 

The main message from these negative results is that a good data distribution alone is not sufficient.
We need to investigate realistic problem-dependent assumptions for MDPs. In the remainder of this work, we explore a restricted class of MDPs, for which we can obtain much stronger guarantees when learning offline, without stringent requirements on data collection.

\section{Action Impact Regularity}
\label{section:air}
Actions play an important role for the exponential lower bound constructions cited in the last section.
They use tree structures where different actions lead to different subtrees and hence different sequence of futures states and rewards. 
A class of MDPs that do not suffer from these lower bounds are those where actions do not have such strong impact on the future states and rewards.
In this section, we introduce the Action Impact Regularity (AIR) property, a property of the MDP which allows for more effective offline RL. The state is partitioned into an exogenous and endogenous component, and the property reflects that the agent's actions primarily impact the endogenous state with limited influence on the exogenous state. We first provide the formal definition and assumptions we leverage to design a practical offline RL algorithm and then discuss when these assumptions are likely to be satisfied.

\subsection{Formal Definition and Assumptions}
We use the standard state decomposition from Exogenous MDPs \cite{mcgregor2017factoring,dietterich2018discovering}. 
We assume the state space is $\cS = \Senv\times\Spri$ where $\Senv$ is the exogenous variable and $\Spri$ is the endogenous variable.
The transition dynamics are $\Penv: \Senv\times\cA\to\Delta(\Senv)$ and $\Ppri: \cS\times\cA\to\Delta(\Spri)$ for exogenous and endogenous variable respectively. The transition probability from a state $s_1=(\env{s_1},\private{s_1})$ to another state $s_2=(\env{s_2},\private{s_2})$ is $P(s_1,a,s_2)=\Penv(\env{s_1},a,\env{s_2})\Ppri(s_1,a,\private{s_2})$. 
\begin{definition}[The AIR Property]
    An MDP is $\varepsilon$-AIR if $\cS = \Senv\times\Spri$, and for any actions $a,a,'\in\cA$, the next exogenous variable distribution is similar if either action $a$ or $a'$ is taken. That is, for each state $s\in\cS$,
    \begin{align*}
        D_{TV}\left(\Penv(\env{s},a),\Penv(\env{s},a')\right)\leq\varepsilon
    \end{align*}
    where $D_{TV}$ is the total variation distance between two probability distributions on $\Senv$. For discrete spaces, the total variation distance is $D_{TV}(P, P') = \frac{1}{2} \norm{P-P'}_1$ ($\ell_1$ norm). 
    \label{def1}
\end{definition}

We define the AIR-MDP such that the property holds for all exogenous state-action pairs. If the property does not hold for one of the exogenous state-action pairs, then one can design an adversarial MDP that hides all difficulties in this single exogenous state-action pair and assuming the properties hold for all but one pair would be useless \cite{jiang2018pac}. 
 
Access to an (approximate) endogenous model is critical to exploit the AIR property, and is a fundamental component of our algorithm. To be precise, we make the following assumption in this paper. 
\begin{assumption}[AIR with an Approximate Endogenous Model]
    We assume that the MDP is $\varair$-AIR and that
    we have the reward model $r:\cS \times \cA \to [0, \rmax]$ and an approximate endogenous model $\HPpri: \cS\times\cA\to\Delta(\Spri)$ such that $D_{TV}(\Ppri(s,a),\HPpri(s,a))\leq\varp$ for any $(s,a)\in\cS\times\cA$.
    \label{ass1}
\end{assumption}

As mentioned in the introduction, it is common to assume that only the transition dynamics are approximated.
Moreover, similar to Definition \ref{def1}, we need the error on the approximate model to hold uniformly. 
Finally, the above assumption implicitly assumes that the separation between exogenous and endogenous state is given to us. More generally, the separation could be identified or learned by the agent, as has been done for contingency-aware RL agents \cite{bellemare2012investigating} and wireless networks \cite{dietterich2018discovering}. Because there are many settings where the separation is clear, we focus on this more clear case first where the separation is known.

\subsection{When Are These Assumptions Satisfied?}

Many real-world problems can be formulated as $\varepsilon$-AIR MDPs. Further, for many of these environments, the separation between exogenous and endogenous state is clear, and we either know or can reasonably approximate the endogenous model. In this section, we go through several concrete examples. 

We can first return to our stock trading example, from the introduction. The exogenous component is the market information (stock prices and volumes) and the endogenous component is the number of stock shares owned by the agent. 
The agent's actions influence their own number of shares, but as an individual trader, have limited impact on stock prices. Using a dataset of stock prices over time allows the agent to reason counterfactually about the impact of many possible trajectories of actions (buying/selling) on its shares (endogenous state) and profits (reward).

There are many settings where the agent has a limited impact on a part of the state. 
The optimal order execution problem is a task to sell $M$ shares of a stock within $H$ steps; the goal is to maximize the profit. The problem can be formulated as an MDP where the exogenous variable is the stock price and endogenous variable is the number of shares left to sell. It is common to assume infinite market liquidity \cite{nevmyvaka2006reinforcement} or that actions have a small impact on the stock price \cite{abernethy2013adaptive,bertsimas1998optimal}; this corresponds to assuming the AIR property.

Another example is the secretary problem \cite{freeman1983secretary}, which a family of problems that can often be used to model real-world application \cite{babaioff2007knapsack,goldstein2020learning}. The goal for the agent is to hire the best secretary out of $H$, interviewed in random order. After each interview, they have to decide if they will hire that applicant, or wait to see a potentially better applicant in the future.
The problem can be formulated as a $0$-AIR MDP where the endogenous variable is a binary variable indicating whether we have chosen to stop or not. 

Other examples include those where the agent only influences energy efficiency, such as in the hybrid vehicle problem \cite{shahamiri2008reinforcement,lian2020rule} and electric vehicle charging problem \cite{abdullah2021reinforcement}.
In the former problem, the agent controls the vehicle to use either the gas engine or the electrical motor at each time step, with the goal to minimize gas consumption; its actions do not impact the driver's behavior. In the latter problem, the agent controls the charging schedule of an electric vehicle to minimize costs; its actions do not impact electricity cost.

In some settings we can even restrict the action set or policy set to make the MDP $\varepsilon$-AIR. For example, if we know that selling $M$ shares hardly impacts the markets, we can restrict the action space to selling less than or equal to $M$ shares. In the hybrid vehicle example, if the driver can see which mode is used, we can restrict the policy set to only switch actions periodically to minimize distractions for the driver. 

In these problems with AIR, we often know the reward and transitions for the endogenous variables, or have a good approximation. For the optimal order execution problem, the reward is simply the selling price times the number of shares sold minus transaction costs, and the transition probability for $\Ppri$ is the inventory level minus the number of shares sold. In other applications, we may be able to use domain knowledge to build an accurate model for the endogenous dynamics. For the hybrid vehicle, we can use domain knowledge to calculate how much gas would be used for a given acceleration. Such information about the dynamics of the system can be simpler for engineers to specify, than (unknown) behavior of different drivers and environment conditions. Our theoretical results will include a term for the error in the endogenous model, but it is reasonable to assume that for many settings we can get that error to be relatively low, particularly in comparison to the error we might get if trying to model the exogenous state.

\subsection{Connections to the Literature on Exogenous MDPs}
AIR MDPs can be viewed as an extension of Exogenous MDPs. 
(1) We allow the action to have small impact on the environmental state, while the action has no impact on the exogenous state in Exogenous MDPs. (2) We do not assume the reward can be decomposed additively to an exogenous reward and an endogenous reward \cite{dietterich2018discovering} nor factor into a sum over each exogenous state variable \cite{chitnis2020learning}. For this previous definition of Exogenous MDPs, the focus was on identifying and removing the exogenous state/noise so that the learning problem could be solved more efficiently \cite{dietterich2018discovering,efroni2021provable}, thus the focus on reward decomposition. Our focus is offline learning where we want to exploit the known structure to enable counterfactual reasoning and avoid data coverage issues.

\section{Offline Policy Optimization for AIR MDPs}

In this section, we discuss several offline algorithms that exploit the AIR property for policy optimization. We then theoretically analyze an FQI-based algorithm, characterizing the performance of its outputted policy.

\subsection{Algorithms for AIR MDPs} \label{sec:algorithm}

Two standard classes of algorithms in offline RL are model-based algorithms---that learn a model from the offline dataset and then use dynamic programming---and model-free algorithms like fitted Q-iteration (FQI). These two approaches can be tailored to our setting with AIR MDPs, as we described below. There is, however, an even more basic approach in our offline RL setting using trajectory simulation, that has previously been used \cite{aiden2020}. We start by describing this simpler approach, and then the modified model-based and FQI approaches.

A natural approach is to reuse trajectories in the dataset to simulate alternative trajectories for an online RL algorithm. For each episode, a random trajectory is selected from the dataset. The online RL algorithm---such as an actor-critic method or a Q-learning agent---takes actions and deterministically transitions to the next exogenous state in the trajectory. The approximate endogenous and reward model are used to sample the next endogenous variable and reward. With such a trajectory simulator, we can run any online reinforcement learning algorithm to find a good policy for the simulator. 

This approach, however, does not exploit the fact that the agent is actually learning offline. The online RL algorithm cannot simply query the model for any state and action, and needs a good exploration strategy to find a reasonable policy. There are fewer theoretical guarantees for such online RL algorithms, and arguably more open questions about their properties than DP-based algorithms and fitted value iteration algorithms. 

A more explicit model-based approach is to learn the exogenous model from data, to obtain a complete transition and reward model, and use any planning approach.
The transition model for exogenous states can be constructed as if the action has no impact.
With the model, we can use any query-efficient planning algorithm to find a good policy for the model. 
Because actions have only small impact in the true MDP, we can learn an accurate 
exogenous model even if we do not have full data coverage. 

More precisely, recall the offline data is randomly generated by running $\pi_b$ on $M$, that is, $D=\{(S_0^{(i)},A_0^{(i)},\dots,S_{H-1}^{(i)},A_{H-1}^{(i)})\}_{i=1}^N$ sampled according to the probability measure $\bbP^{\pi_b}_M$. 
The pertinent part is the transitions between exogenous variables, so we define $D_{\tenv} =\{(S_0^{(i)},S_{1}^{(i),\tenv},\dots,S_{H-1}^{(i),\tenv})\}_{i=1}^N$.
The model-based approach constructs an empirical MDP $M_D=(\cS,\cA,\hat P,r,H,\hat \nu)$. For the tabular setting we have $\hat \nu(s) = \frac{1}{N} \sum_{i=1}^N \II(S_0^{(i)}=s)$, and
\begin{equation*}
    \HPenv(\senv_h,a,\senv_{h+1}) \!=\! \frac{\!\sum_{i=1}^N \II(S_h^{(i),\tenv}\!=\senv_h\!\!,S_{h+1}^{(i),\tenv}\!=\senv_{h+1})}{\sum_{i=1}^N \II(S_h^{(i),\tenv}\!=\senv_h)}
\end{equation*} 
for all $a \in \cA$. Exogenous variables not seen in the data are not reachable, and so can either be omitted from $\HPenv$ or set to self-loop. 
For large or continuous state spaces, we can learn $p(\senv_{h+1}|\senv_h)$ using any conditional distribution learning algorithm, and set $\HPenv(\senv_h,a,\senv_{h+1}) = p(\senv_{h+1}|\senv_h)$ for all $a \in \cA$.

For large or continuous states spaces, however, learning such a model and planning can be impractical. Learning an accurate exogenous model might be difficult if the exogenous transition is complex or the exogenous state is high-dimensional. Further, it is not possible to sweep through all states during planning. Smarter approximate dynamic programming algorithms need to be used, but even these can be quite computationally costly.

A reasonable alternative is FQI, which approximates value iteration without the need to learn a model.
Our FQI algorithm that exploits the AIR property is described in Algorithm \ref{fqi-air}, which we call FQI-AIR.
The algorithm simulates all actions from a state, and assumes it transitions to the exogenous state observed in the dataset. 
The reward and endogenous state for each simulated action can be obtained using the reward model and approximate endogenous model. 
Even though the true MDP is not necessarily $0$-AIR MDP, we will show in the analysis that as long as $\varair$ is small, the algorithm can return a nearly optimal policy in the true MDP. 
This algorithm, although simple, enjoys theoretical guarantees without making assumptions on the concentration coefficient, and can be much more computationally efficient than trajectory simulation methods.

\begin{algorithm}[t]
    \caption{FQI-AIR}
    \label{fqi-air}
    \begin{algorithmic}
        \STATE Input: dataset $D$, value function class $\calF$, $\HPpri$, $r$
        \STATE Let $q_H=0$, $D_{H-1} = \emptyset$, ..., $D_{1} = \emptyset$
        \FOR {$h=H-1,\dots,0$}
            \STATE For all $i\in\{1,\dots,N\}$, all $\private{s_h} \in \Spri$, all $a \in \cA$
            \STATE Sample $s^{'\tpri} \sim \HPpri(s_{h}^{(i),\tenv},\private{s_h},a)$, compute target
            \begin{align*}
                t =  r(s_{h}^{(i),\tenv},\private{s_h},a) + \max_{a'\in\cA} q_{h+1}(s_{h+1}^{(i),\tenv},s^{'\tpri},a')
            \end{align*}
            \STATE Add (synthetic) pair $((s_{h}^{(i),\tenv},s^\tpri,a),t)$ to $D_h$
            \STATE After generating $D_h$
            \begin{align*}
                & q_h = \arg\min_{f\in\calF} \sum_{(x,y) \in D_h}(f(x) - y)^2 \\ 
                & \pi_h(s) = \arg\max_{a} q_h(s,a) \text{ for all } s\in\cS_h
            \end{align*}
        \ENDFOR
        \STATE Output: $\pi=(\pi_0,\dots,\pi_{H-1})$
    \end{algorithmic}
\end{algorithm}

Note that the computational cost scales with the size of $\Spri$ and $\cA$. When $|\Spri|$ or $|\cA|$ is large, we can modify FQI-AIR to no longer use full sweeps. Instead, we can randomly sample from the endogenous state space and action.
We include a practical implementation of FQI-AIR in Algorithm \ref{fqi-air-practical}. 
For each exogenous state in the dataset, we sample an endogenous state and an action, and query the approximate model to obtain a target for FQI update.
As a result, the computation can be independent of the size of $\Spri$ and $\cA$.
However, for sample complexity, the performance loss of the algorithm would depend on the squared root of the size $|\Spri||A|$, as shown in the next section.

\begin{algorithm}[t]
    \caption{FQI-AIR for large state spaces}
    \label{fqi-air-practical}
    \begin{algorithmic}
        \STATE Input: dataset $D$, an approximate model $\HPpri$, $r$, mini-batch size $B$, number of training iteration $K$, number of updates per iteration $M$
        \STATE Initialize a Q function $q_\theta:\Senv\times\Spri\times\cA\times[H]\to\RR$, parameterized by $\theta$
        \STATE $\bar\theta\gets\theta$ 
        \FOR {$k=1,\dots,K$}
            \FOR{$m=1,\dots,M$}
            \STATE Sample a mini-batch of transitions $\{(\senv_j,\spri_j,a_j,h_j,s_j^{'\tenv}, s_j^{'\tpri})\}_{j=1}^B$ from $D$
            \STATE For all $j$, sample an endogenous state $\tilde s^\tpri_j \in\Spri$ and an action $\tilde a_j\in\cA$ randomly,  sample $\tilde{s}^{'\tpri}_j \sim \HPpri(s_{j}^{\tenv},\tilde s^\tpri_j,\tilde a_j)$, and compute target
            \begin{align*}
                t_j =  r(\senv_j,\tilde s^\tpri_j,\tilde a_j) + \max_{a'\in\cA} q_{\bar\theta}(s_j^{'\tenv},\tilde s^{'\tpri}_j,a',h_j+1)
            \end{align*}
            \STATE Compute the mini-batch loss $L(\theta) = \sum_{j=1}^B (q_\theta(\senv_j,\tilde s^\tpri_j,\tilde a_j,h_j) - t_j)^2$
            \STATE Update $\theta$ to reduce $L(\theta)$ 
            \ENDFOR
            \STATE $\bar\theta\gets\theta$ 
        \ENDFOR
        \STATE Output: the greedy policy with respect to $q_\theta$
    \end{algorithmic}
\end{algorithm}

\subsection{Theoretical Analysis of FQI-AIR}\label{sec_main_theory}
First we need the following definitions. For a given MDP $M$, we define $J(\pi, M) \defeq \EE^\pi_M[R(\tau)]$ where $\tau=(S_0,A_0,\dots,S_{H-1},A_{H-1})$ is a random element in $(\cS\times\cA)^H$, the expectation $\EE^\pi_M$ is with respect to $\bbP^\pi_M$, and $R(\tau)=\sum_{h=0}^{H-1} r(S_h,A_h)$.

We also need the following assumption on the function approximation error. This is a common assumption to analyze approximate value iteration algorithms \cite{antos2006learning,munos2007performance}. 
Let $\tilde \nu_h(\env{s_h})=\bbP^{\pi_b}_M(\env{S_h}=\env{s_h})$ be the data distribution on $\Senv$ at horizon $h$. Given a probability measure $\nu_h$ on $\Senv$ and $p\in [1,\infty)$, define the norm
\begin{align*}
    &\norm{q}_{p, \nu_h}^p = \sum_{\env{s} \in \env{S}} \sum_{\private{s}\in\private{S}} \sum_{a\in\cA} \frac{\nu_h(\senv)}{|\Spri||\cA|} |q(\senv,\spri,a)|^p.
\end{align*}

\begin{assumption}
    Assume the function class $\cF$ is finite and the inherent Bellman error is bounded by $\varapx$, that is,
    \begin{align*}
        \varapx = \max_{h\in[H]} \max_{f'\in\cF} \min_{f\in\cF}  \norm{\bellman f' - f}_{2,\tilde \nu_h}^2.
    \end{align*}
    \label{ass2}
\end{assumption}
We assume the function class is finite for simplicity, which is common in many offline RL papers \cite{chen2019information}. If the function class is not finite but has a bounded complexity measure, we can derive similar results by replacing the size of the function class with the complexity measure. For example, \citeA{duan2021risk} analyze FQI with the Rademacher complexity. Since studying the complexity measure is not a critical point for our paper, we decide to make the finite function class assumption.

\begin{restatable}{theorem}{theoremfqi}
    \label{theoremfqi}
    Under Assumption \ref{ass1} and \ref{ass2}, let $\pi^*_M$ be an optimal policy in $M$ and $\pi$ the output policy of FQI-AIR, then with probability at least $1-\zeta$ 
    \begin{align*}
        &J(\pi, M) \geq J(\pi^*_M, M) - 2 \vmax H (\varair+\varp) - \\
        &(H+1)H\sqrt{ |\Spri||\cA|} \left(\sqrt{\frac{72\vmax^2\ln{(H|\cF||\Spri||\cA|/\zeta)}}{N} + 2\varapx}\right).
    \end{align*}
\end{restatable}

The bound on performance loss has three components: (1) a sampling error term which decreases with more trajectories; (2) the AIR parameter; and (3) an approximation error term which depends on the function approximation used. The result implies that as long as we have a sufficient number of episodes, a good function approximation, and small $\varair$, then the algorithm can find a nearly-optimal policy with high probability. For example, if $\varair,\varp$ and $\varapx$ are small enough, we only need $N=\tilde O(H^4\vmax^2|\Spri||\cA|/\delta)$ trajectories, which is polynomial in $H$, to obtain a $\delta$-optimal policy.

The proof can be found in the appendix. The key idea is to introduce a baseline MDP $M_b$ that is $0$-AIR, that approximates $M$ which is actually $\varair$-AIR. The baseline MDP $M_b=(\cS,\cA,\Tilde P,r,H,\nu)$ has $\TPenv(\env{s_h},a,\senv_{h+1}) = \bbP^{\pi_b}_M(\env{S_h}=\env{s_h},\env{S_{h+1}}=\env{s_{h+1}})/\bbP^{\pi_b}_M(\env{S_h}=\env{s_h})$ and $\TPpri(s_h,a,\private{s_h})=\HPpri(s_h,a,\private{s_h})$.
The transition probability for exogenous state does not depend on the action $a$ taken, so $M_b$ is $0$-AIR. We show that FQI returns a good policy in $M_b$, and that good policies in $M_b$ are also good in the true MDP $M$.

We can contrast this bound to others in offline RL. 
For FQI results that assume the concentration coefficient is bounded and small, the error bound has a term that scales with $\sqrt{C}$, which is on the same order as the term $\sqrt{|\Spri||\cA|}$ in our bound. 
We can get a similar bound by considering this restricted class of MDPs that are $\varepsilon$-AIR, without having to make any assumptions on the concentration coefficient. 
For settings where this assumption is appropriate---namely when the MDP is $\varepsilon$-AIR---this is a significant success, as we need not make these stringent conditions on data distributions.   %

\section{Policy Evaluation for AIR MDPs}
We can also exploit the AIR property, and access to the approximate endogenous model and reward model, to evaluate the value of a given policy. Given a trajectories of exogenous states  $(S_0^{(i)},S_1^{(i),\tenv},\dots,S_{H-1}^{(i),\tenv})$, we can rollout a synthetic trajectory under $\pi$ and $\HPpri$: $\tau_{D}^{(i)}=(S_0^{(i)},A_0^{(i)},S_1^{(i)},A_1^{(i)},\dots)$ where $A_t^{(i)}\sim\pi(S_t^{(i)})$, $S_{t+1}^{(i),\tpri}\sim\HPpri(S_t^{(i)},A_t^{(i)})$ and $S_{t+1}^{(i)} = [S_{t+1}^{(i),\tenv},S_{t+1}^{(i),\tpri}]$. For $R(\tau_D^{(i)})\defeq\sum_{t=0}^{H-1} r(S_t^{(i)}, A_t^{(i)})$, the average return over the $N$ synthetic trajectories $\hat J(\pi,M)=\frac{1}{N}\sum_{i=1}^N R(\tau_D^{(i)})$ is an estimator of $J(\pi,M)$. This method is simple, but very useful because we can do hyperparameter selection with only the offline dataset without introducing extra hyperparameters.

We can bound the policy evaluation error by Hoeffding's inequality. More sophisticated bounds for policy evaluation can be found in \cite{thomas2015high}.
\begin{restatable}{theorem}{theoremope}
    \label{theoremope}
    Under Assumption \ref{ass1}, given a deterministic policy $\pi$, we have that with probability at least $1-\zeta$
    \begin{equation*}
        \!\left|\hat J(\pi,M) \!-\! J(\pi, M)\right| \!\leq\!\! 
        \vmax \!\left(H \varair \!+\! H \varp \!+\! \sqrt{\tfrac{\ln{(2/\zeta)}}{2N}}\right).
    \end{equation*}
\end{restatable}
The results suggests that if we have a sufficient number of trajectories and small $\varair$ and $\varp$, then $\hat J(\pi,M)$ is a good estimator for $J(\pi, M)$. Even though the estimator is biased and not consistent, we find it provides sufficient information for hyperparameter selection in our experiments.

Theorem \ref{theoremope} only holds for a given policy that is independent of the offline data. If we want to evaluate the output policy, which depends on the data, we need to apply an union bound for all deterministic policies. In that case, the sampling error term becomes $\tilde O(\sqrt{|\cS|/N})$. To avoid the dependence on the state space, an alternative is to split the data into two subsets: one subset is used to obtain an output policy and another subset is used to evaluate the output policy. 

\section{Simulation Experiments}
\label{section:exp}

We evaluate the performance of the FQI-based algorithm in two simulated environments with the AIR property: an optimal order execution problem and an inventory management problem.  
We also tested the other algorithms described in Section \ref{sec:algorithm}, for completeness and to contrast to FQI-AIR. We include these results in Appendix \ref{appendix:other_air}. FQI-AIR is notably better than these other approaches, and so we focus on it as the representative algorithm that exploits the AIR property.

The first goal of these experiments is to demonstrate that existing offline RL algorithms fail to learn a good policy for some natural data collection policies, while our proposed algorithm returns a near-optimal policy. 
To demonstrate this, we test three data collection policies: (1) a random policy which is designed to give a small concentration coefficient, (2) a learned near-optimal policy obtained using DQN with online interaction, which covers an optimal policy reasonably well, and (3) a constant policy which, in theory, has an infinite concentration coefficient due to missing state-action pairs and does not cover an optimal policy. The second goal is to validate the policy evaluation analysis with a varying number of trajectories $N$ and $\varair$. 

\subsection{Environments}

We investigated the behavior of the algorithms on two simulated environments that mimic two real-world problems that satisfy the AIR property: optimal order execution and inventory management. In the \textbf{optimal order execution problem}, the task is to sell $M=10$ shares within $H=100$ steps. The stock prices $X_1,\dots,X_h$ are generated by an ARMA($2,2$) process and scaled to the interval $[0,1]$. Specifically, the ARMA($2,2$) process is
\begin{align*}
    X_t &= \Tilde X_t/20 + 1/2 \quad \quad \quad\text{where } \quad
    \Tilde X_t = c + \varepsilon_t + \sum_{i=1}^2 \varphi_i \Tilde X_{t-i} + \sum_{i=1}^2 \theta_i \varepsilon_{t-i},
\end{align*}
$\varphi_1\sim\mathcal{U}(-0.9,0.0)$, $\varphi_2\sim\mathcal{U}(0.0,0.9)$ and $\theta_i\sim\mathcal{U}(-0.5,0.5)$ for $i=1,2$. The scaling parameters are chosen so that the process is stable and the price is in the interval $[0,1]$. 
The endogenous variable $P_h$ is the number of shares left to sell. To construct state, we use the most recent $K=3$ prices with the most recent endogenous variable, that is, $S_h=(X_{h-2}, X_{h-1}, X_{h}, P_h)$. The action space is $\cA=[5]$. The reward function is the stock price $X_h$ multiplied by the number of selling shares $\min\{A_h, P_h\}$. 

We consider both a setting with $\varair = 0$ and $\varair > 0$, as well as different data generating policies. When the number of selling shares is greater than $0$, the stock price drops by $10\%$ with probability $\varair$. For $\varair = 0$, this means selling shares has no impact on the stock price. When $\varair > 0$, it does, allowing us to test how robust FQI-AIR is to some violation of the AIR property.   
The random policy used in the environment chooses $0$ with probability 75\% and choose $1,\dots,5$ with probability 5\%. The constant policy always chooses action $0$. 

We design an \textbf{inventory management problem} based on existing literature \cite{kunnumkal2008using,van1997neuro}. The task is to control the inventory of a product over $H=100$ stages. At each stage, we observe the inventory level $X_t$ (endogenous) and the previous demand $D_{t-1}$ (exogenous) and choose to place a order $A_t\in[10]$. The inventory level evolves as: $X_{t+1} = (X_t + A_t - D_t)^+$. The reward function is $c A_t + h (X_t + A_t - D_t)^{+} - b (-X_t - A_t + D_t)^{+}$ where $c$ is the order cost, $h$ is the holding cost, and $b$ is the cost for lost sale. 
We use $c=0.1$, $h=0.25$ and $b=1.0$ in the experiment. To make sure the reward is bounded, we clip the reward at a large negative number $-100$. 
The endogenous variable is the inventory level, which can be as large as $1000$, so we restrict FQI-AIR to sweep only for a subset of the endogenous space, that is, $[15]\subset\Spri$.

As before, we consider both $\varair = 0$ and $\varair > 0$, which in this case impacts the demand. The demand $D_t=(\Tilde D_t)^+$ where $\Tilde D_t$ follows a normal distribution with mean $\mu$ and $\sigma=\mu/3$ and $(d)^{+}\defeq\max\{d,0\}$. In the beginning of each episode, $\mu$ is sampled from a uniform distribution in the interval $[3,9]$. 
When the order is greater than $0$, the mean of the demand distribution decreases or increases by $10\%$ with probability $\varair/2$ respectively. 
  
Again, we consider three different data generating policies. The random policy used in the environment is a policy which chooses a value $\Tilde A_t \in [D_{t-1}-3,D_{t-1}+3]$ uniformly and then choose the action $A_t=\max\{\min\{\Tilde A_t,10\},0\}$. The constant policy always chooses action $A_t = \min\{D_{t-1},10\}$. The near-optimal policy is obtained using DQN with online interaction, for both environments.

There is an important nuance for the inventory problem. In this problem, the endogenous transition and reward depends on the next exogenous variable. Fortunately, we can generalize the definition of exogenous MDPs such that the endogenous transition is $\Ppri: \cS\times\cA\times\Senv\to\Delta(\Spri)$ and the reward is $r:\cS\times\cA\times\Senv\to[0,\rmax]$. We assume we have an approximate endogenous model $\HPpri: \cS\times\cA\times\Senv\to\Delta(\Spri)$ such that $D_{TV}(\Ppri(s,a,\senv),\HPpri(s,a,\senv))\leq\varp$ for any $(s,a,\senv)\in\cS\times\cA\times\Senv$. With these changes, the algorithms and the theoretical analysis naturally extend to the new definition of exogenous MDPs.

\subsection{Algorithm Details}
We compare FQI-AIR to FQI, MBS-QI \cite{liu2020provably}, CQL \cite{kumar2020conservative}, IQL \shortcite{kostrikov2021offline}.
As we discussed in the previous sections, FQI is expected to work well when the concentration coefficient is small. MBS-QI is expected to perform well when the data covers an optimal policy. CQL and IQL are strong baselines which have been shown to be effective empirically for discrete-action environments such as Atari games. 

We had several choices to make for the baseline algorithms. 
MBS-QI requires density estimation for the data distribution $\mu$. For the optimal order execution problem, we use state discretization and empirical counts to estimate the data distribution as used in the original paper. For the inventory problem, the state space is already discrete so there is no need for discretization. We show the results with the best threshold $b$ from the set $\{0.002, 0.001, 0.0001, 0.00005\}$. 
Note that it is possible that there is no data for some states (or state discretization) visited by the output policy, and for these states, all action values are all zero. To break ties, we allow MBS-QI to choose an action uniformly at random.
For CQL, we add the CQL($\mathcal{H}$) loss with a weighting $\alpha$ when updating the action values. We show the results with the best $\alpha$ from the set $\{0.1,0.5, 1.0, 5.0\}$ as suggested in the original paper. 
For IQL, we show the results with the best $\tau$ from the set $\{0.7,0.8,0.9\}$ and $\beta$ from the set $\{10.0, 3.0, 1.0\}$.

We use the same value function approximation for all algorithms in our experiments: two-layers neural networks with hidden size 128. The neural networks are optimized by Adam \cite{kingma2014adam} or RMSprop with learning rate in the set $\{0.001, 0.0003, 0.0001\}$. All algorithms are trained for 100 iterations. We also tried training the comparator algorithms for longer, but it did not improve their performance. 

The hyperparameters for FQI-AIR are selected based on $\hat J(\pi,M)$, which only depends on the dataset. The hyperparameters for comparator algorithms are selected based on $J(\pi,M)$---which should be a large advantage---estimated by running the policy $\pi$ on the true environment $M$ with $100$ rollouts.

\subsection{Policy Performance When $\varair = 0$}

\begin{figure}[t]
    \centering
    \subfigure[Optimal order execution problem]
    {
        \centering
        \includegraphics[width=0.8\textwidth]{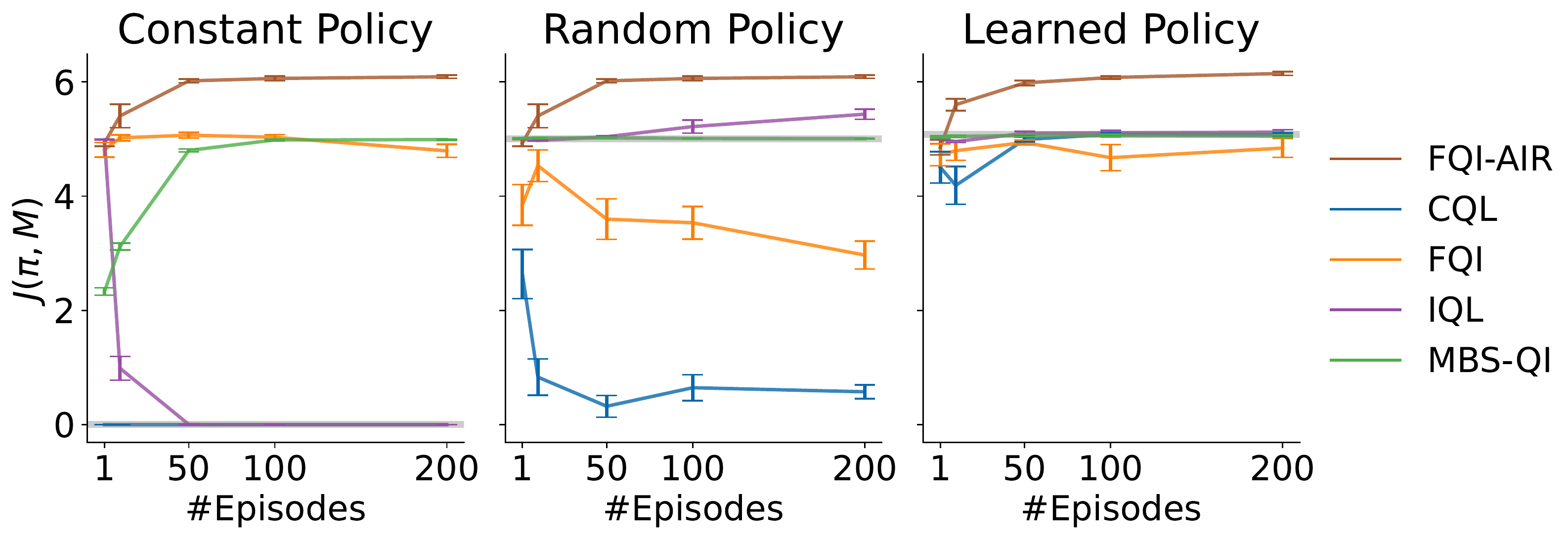}
    }
    \subfigure[Inventory management problem]
    {
        \centering
        \includegraphics[width=0.8\textwidth]{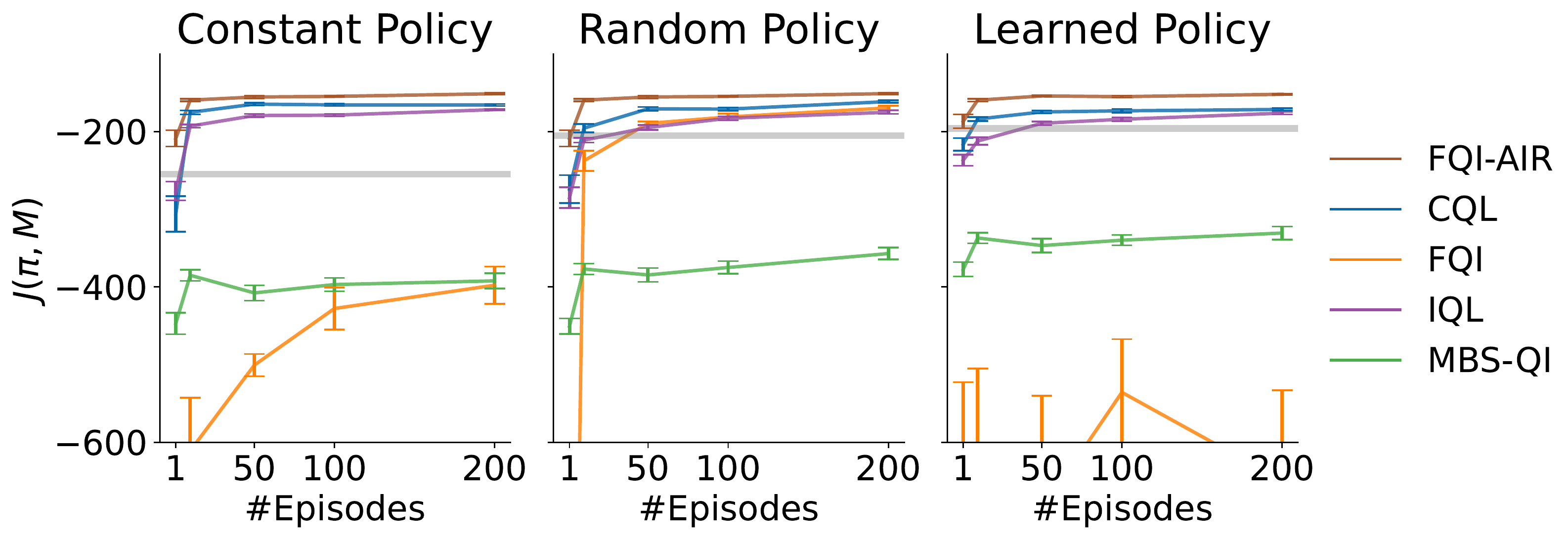}
    }
    \caption{Comparison between algorithms in the optimal order execution problem and inventory management problem, for $\varair= 0$. The gray lines show the performance of the data collection policies. Results are averaged over 30 runs with error bars for one standard error.}
    \label{sim_result}
\end{figure}

Figure \ref{sim_result} shows the performance of our algorithm and the comparator algorithms with a different number of trajectories $N=\{1,5,10,25,50,100,200\}$ and $\varair=0$. Our algorithm outperforms other algorithms for all data collection policies. This result is not too surprising, as FQI-AIR is the only algorithm to exploit this important regularity in the environment; but nonetheless it shows how useful it is to exploit this AIR property when it is possible.

We can first look more closely at the optimal order execution results. MBS performs slightly better than FQI, however, we found it performs better because the tie-breaking is done with a uniform random policy especially under the constant policy dataset.\footnote{A uniform policy in this environment can achieve a performance $J(\pi,M)\approx 5$.}
CQL and IQL fails when the data collection policy is far from optimal (constant policy) and perform reasonably with a learned policy. FQI-AIR exploits the AIR property, and so is robust to different data collection policies.
The results show that exploiting the AIR property is critical for the robust performance.

We see similar patterns for the inventory management problem. FQI-AIR outperforms the other algorithms for all data collection policies. CQL and IQL perform well in this environment. MBS outperforms FQI under the learned policy, but FQI outperforms MBS under the random policy. The results match the expectation that FQI performs well with an uniform data and MBS-QI performs well with an expert data.

\subsection{Policy Performance with a Large Value of $\varair$}
Now we consider the impact of using these algorithms when $\varair > 0$. We should expect FQI-AIR to be most impacted, as the other algorithms do not exploit the AIR property. We vary $\varair$ from $0.1$ to $0.8$ and find that the results are similar to those with small $\varair$. FQI-AIR still significantly outperforms other offline methods.

\begin{figure}[t]
    \centering
    \subfigure[Optimal order execution problem]
    {
        \centering
        \includegraphics[width=0.8\textwidth]{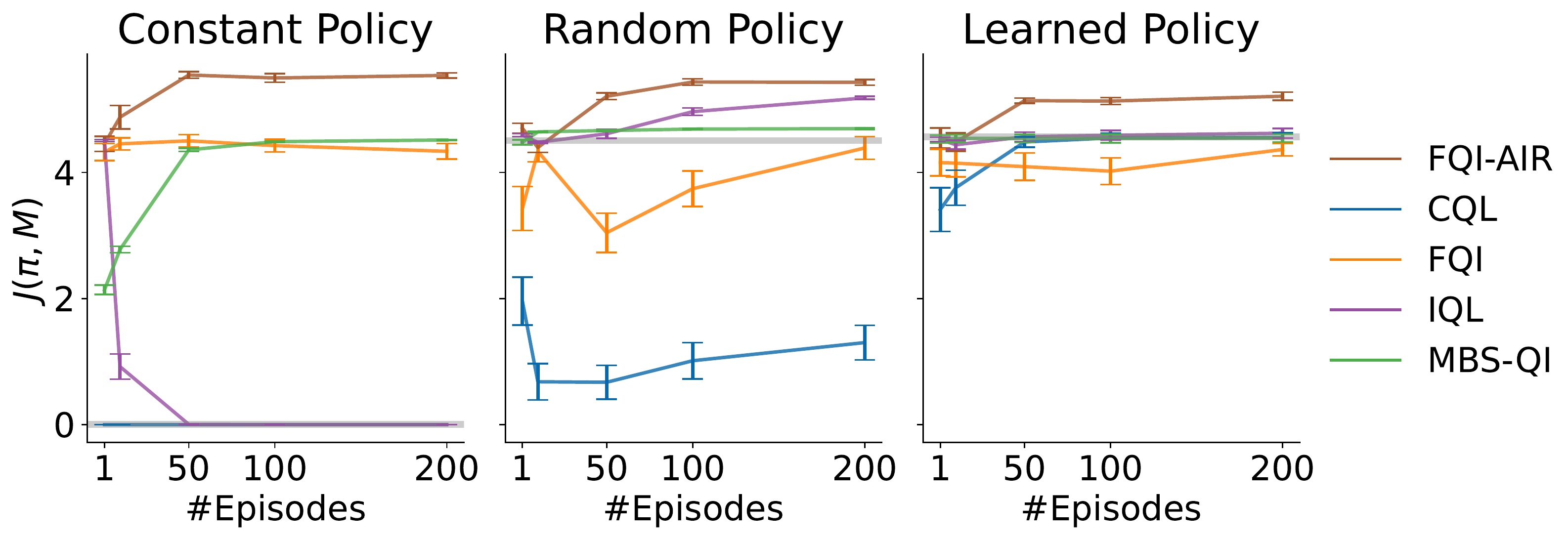}
    }
    \subfigure[Inventory management problem]
    {
        \centering
        \includegraphics[width=0.8\textwidth]{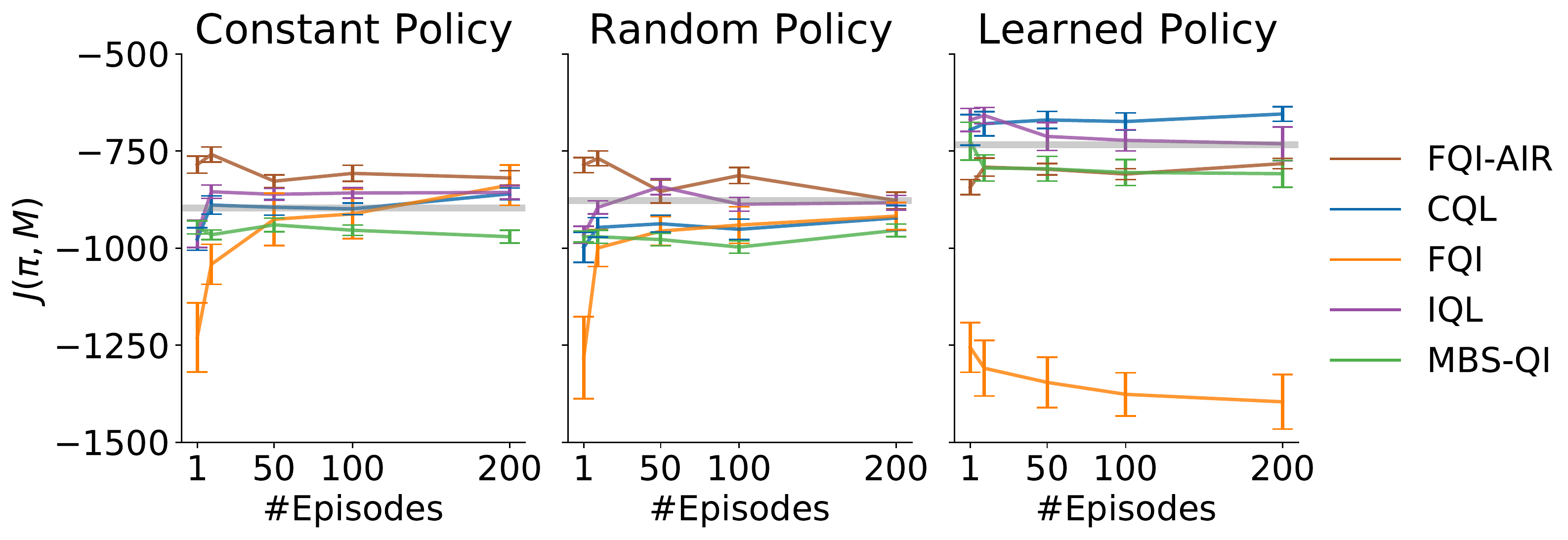}
    }
    \caption{Comparison between algorithms in two simulation problems with $\varair=0.8$. The gray lines show the performance of the data collection policies. The results are averaged over 30 runs with error bars representing one standard error.}
    \label{exp_air_0.1}
\end{figure}

Figure \ref{exp_air_0.1} shows the result with $\varair = 0.8$ where the performance of all algorithms drop significantly. In theory, FQI-AIR can have a large performance loss with large $\varair$, however, FQI-AIR still consistently outperforms other baselines in our experiments, except for the inventory management problem with the learned policy. This is because the divergence between the true exogenous transition and the synthetic exogenous transition in FQI-AIR does not occur at every time step even when $\varair$ is large. For example, in the optimal order execution problem, the divergence can only happen when we sell a positive number of shares. 
The theoretical result is the worst-case analysis where the divergence can occur at every time step and we suffer $\rmax$ loss every time the divergence occurs. Therefore, the experiment results suggest that these practical problems considered in the paper are not the worst case and FQI-AIR can perform well even with large $\varair$.

\subsection{Results for Policy Evaluation}\label{sec_validate}
To validate the policy evaluation analysis, we investigate the difference $|\hat J(\pi,M) - J(\pi, M)|$ with $\varair\in\{0,0.05,0.1,0.2, 0.4\}$ and $N=\{1,5,25,100,200\}$ where $\pi$ is the output policy of FQI-AIR. 
We show the 90th percentile of the difference for each combination of $\varair$ and $N$ over 90 data points (30 runs under each data collection policy) in Figure \ref{lineplot}. The 90th percentiles scale approximately linearly with $\varair$ and inversely proportional to $N$. The results suggest that the dependence on $\varair$ is linear and the policy evaluation error goes to zero at a sublinear rate, which reflects the bound of Theorem \ref{theoremope}.

\begin{figure}[t]
    \centering
    \subfigure[Optimal order execution problem]
    {
        \centering
         \includegraphics[width=0.45\textwidth]{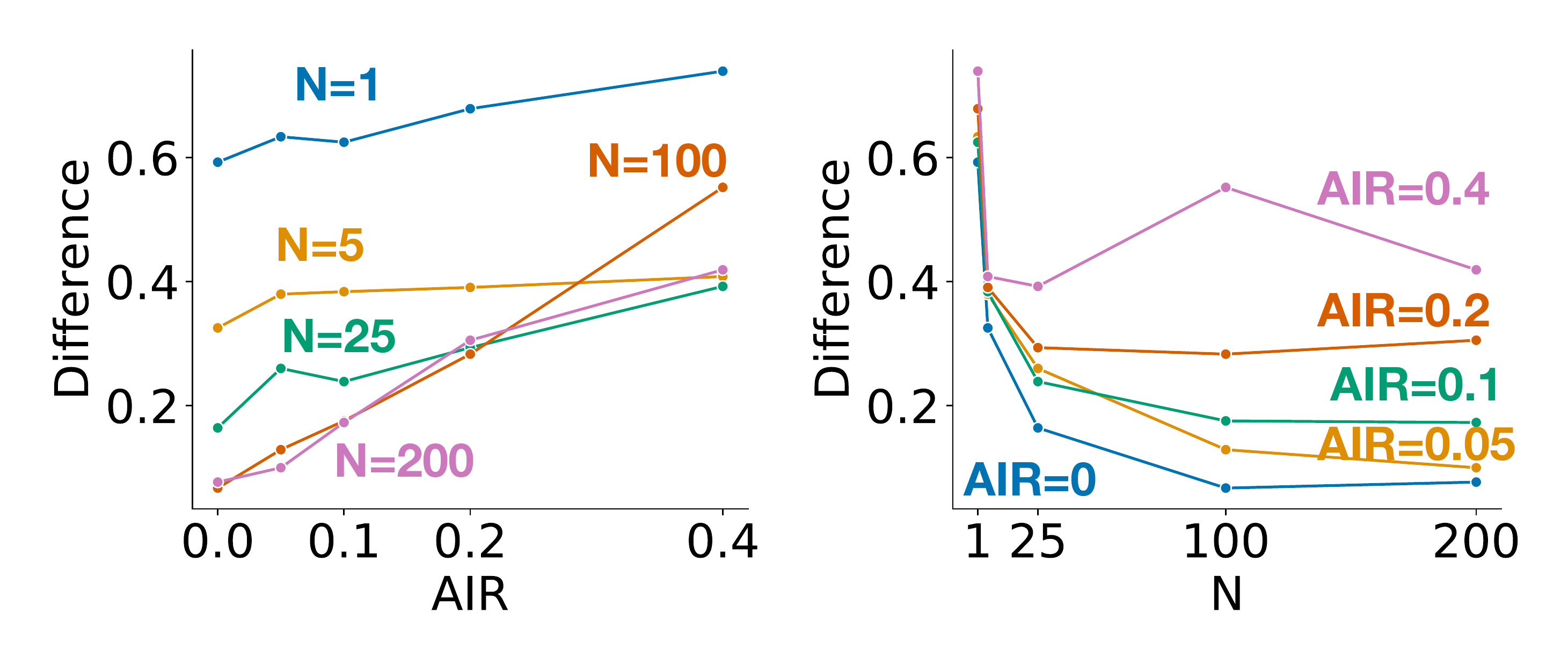}
    }
    \subfigure[Inventory management problem]
    {
        \centering
        \includegraphics[width=0.45\textwidth]{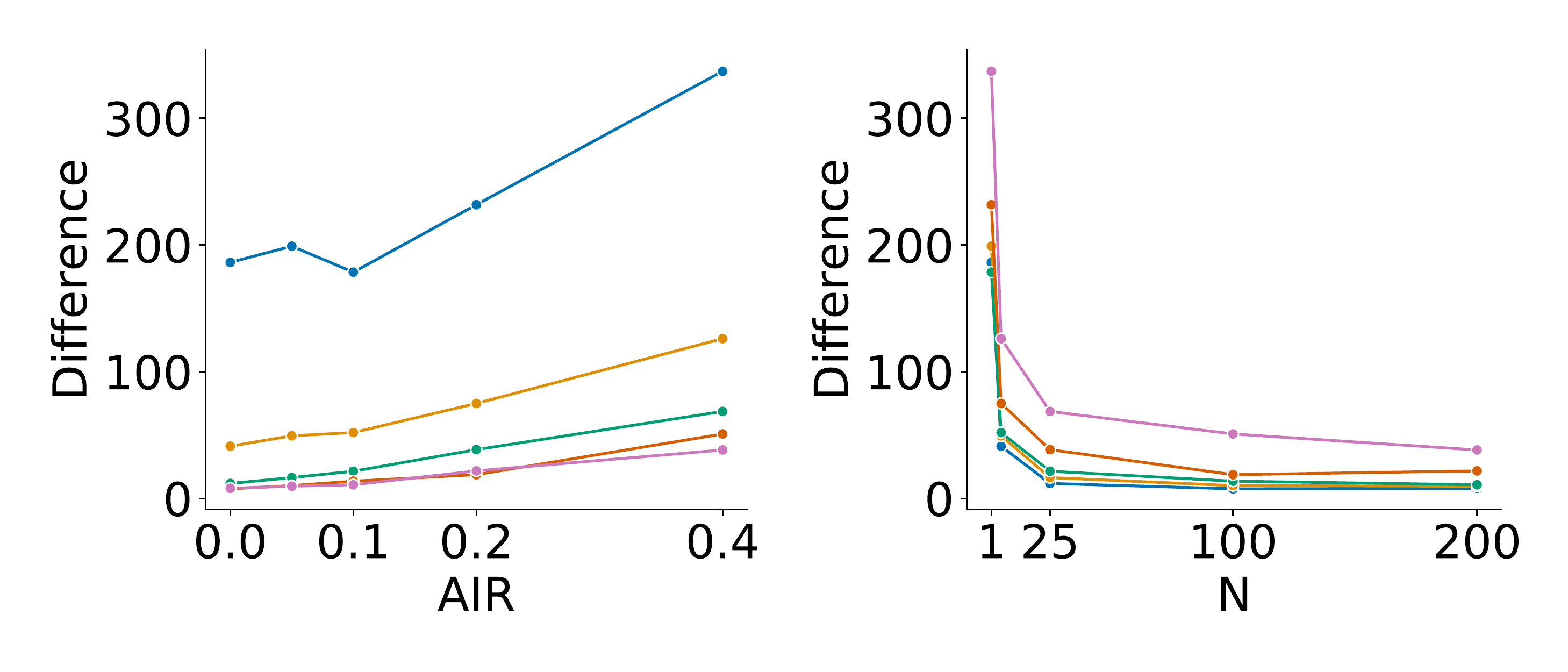}
    }
    \caption{The 90th percentile, over 90 runs, of the difference $|\hat J(\pi,M) - J(\pi, M)|$, with varying $N$ and $\varair$. The difference should be lower for a larger $N$ and smaller $\varair$.
     }
    \label{lineplot}
\end{figure}

\section{Real World Experiments}
To demonstrate the practicality of the proposed algorithm, we evaluate the proposed algorithm for two real world experiments: (1) \emph{Bitcoin}: an optimal order execution for the bitcoin market, and (2) \emph{Prius}: a hybrid car control problem. 
For the Bitcoin experiment, we use historical prices of bitcoin.\footnote{The bitcoin data is downloaded from the kaggle competition \url{https://www.kaggle.com/c/g-research-crypto-forecasting/data}.} The problem is to sell one bitcoin within 60 steps where each step corresponds to 10 minutes in real world. On each step, the agent chooses to sell some numbers of bitcoins in $\{0,0.1,0.2,0.3,0.4,0.5\}$. 
Each episode corresponds to 10 hours, with a start state chosen from a random time step in the data (consisting of 300~days).  The exogenous state contains the most recent 60 closing prices, and the endogenous state contains the number of shares left to be sold. We collect an offline dataset by running a trained policy by DQN for $N$ episodes, and report performance of the output policy for the testing period (about 41~days). 

For the Prius experiment, we use the hybrid car environment from \citeA{lian2020rule}.\footnote{We used their code at \url{https://github.com/lryz0612/DRL-Energy-Management/blob/master/Prius_model_new.py}.} The agent can switch between using fuel or battery, with the goal to minimize fuel consumption while maintaining a desired battery level. The exogenous state is the driving patterns and the endogenous state contains the state of charge and the previous action. We collect the offline dataset by running a learned policy with 10 different driving patterns, and test on 12 driving patterns. To better mimic the real-world, where we would not have a random policy or constant policy, we use the learned policy from DQN as the data collection policy. Further, now that the state space is larger, we run FQI-AIR where we randomly sample endogenous states and actions, rather than sweeping through all endogenous states and actions.

FQI-AIR performs significantly better than CQL, IQL and FQI, as shown in Figure \ref{exp:realworld}. MBS-QI does not scale to high-dimensional continuous state spaces, and so is excluded. These results highlight that FQI-AIR can scale to high-dimensional continuous state space and large endogenous state space. 

\begin{figure}[t]
    \centering
    \subfigure[Bitcoin]
    {
        \centering
         \includegraphics[scale=0.34]{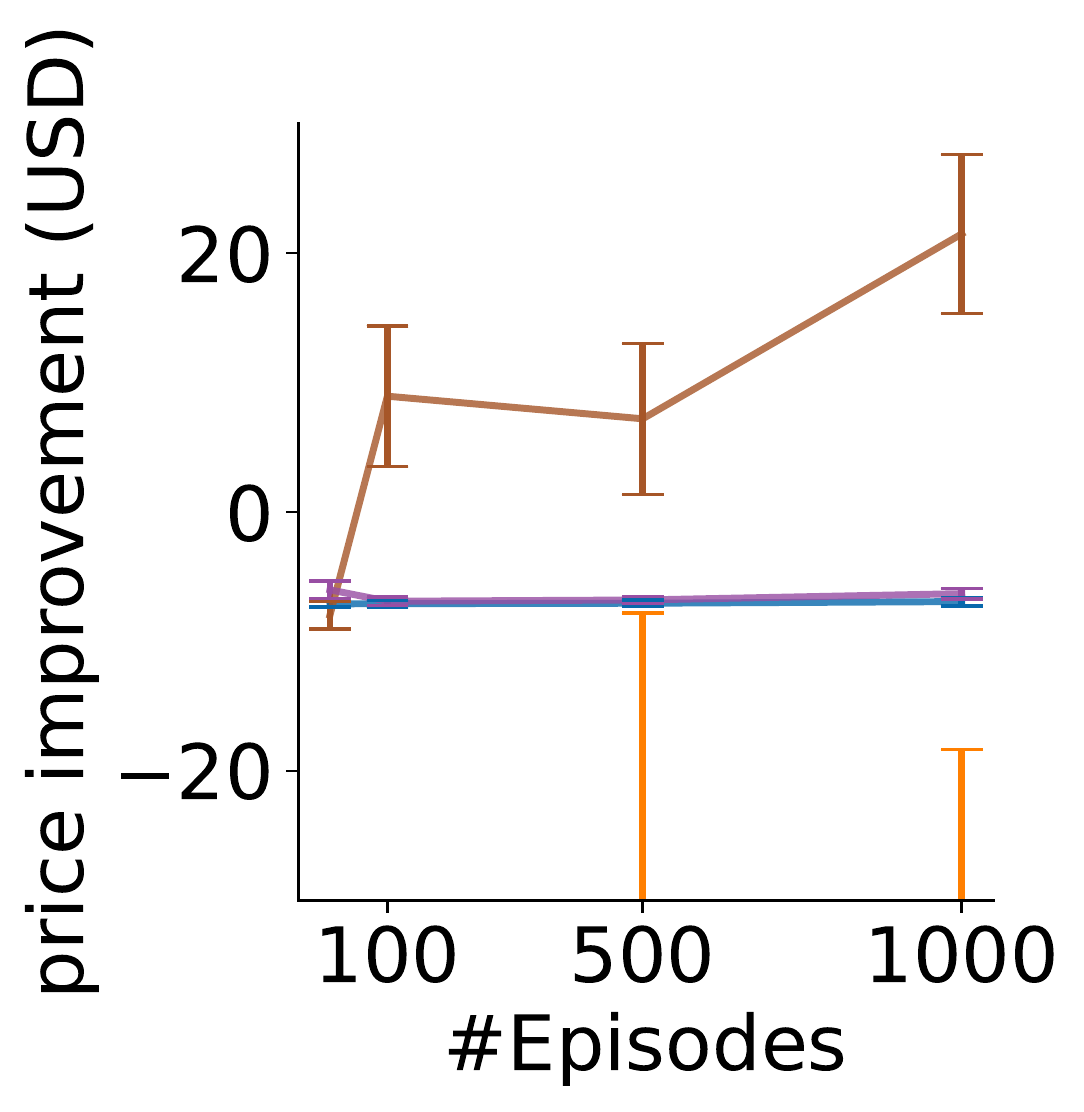}
    }
    \subfigure[Prius]
    {
        \centering
        \includegraphics[scale=0.34]{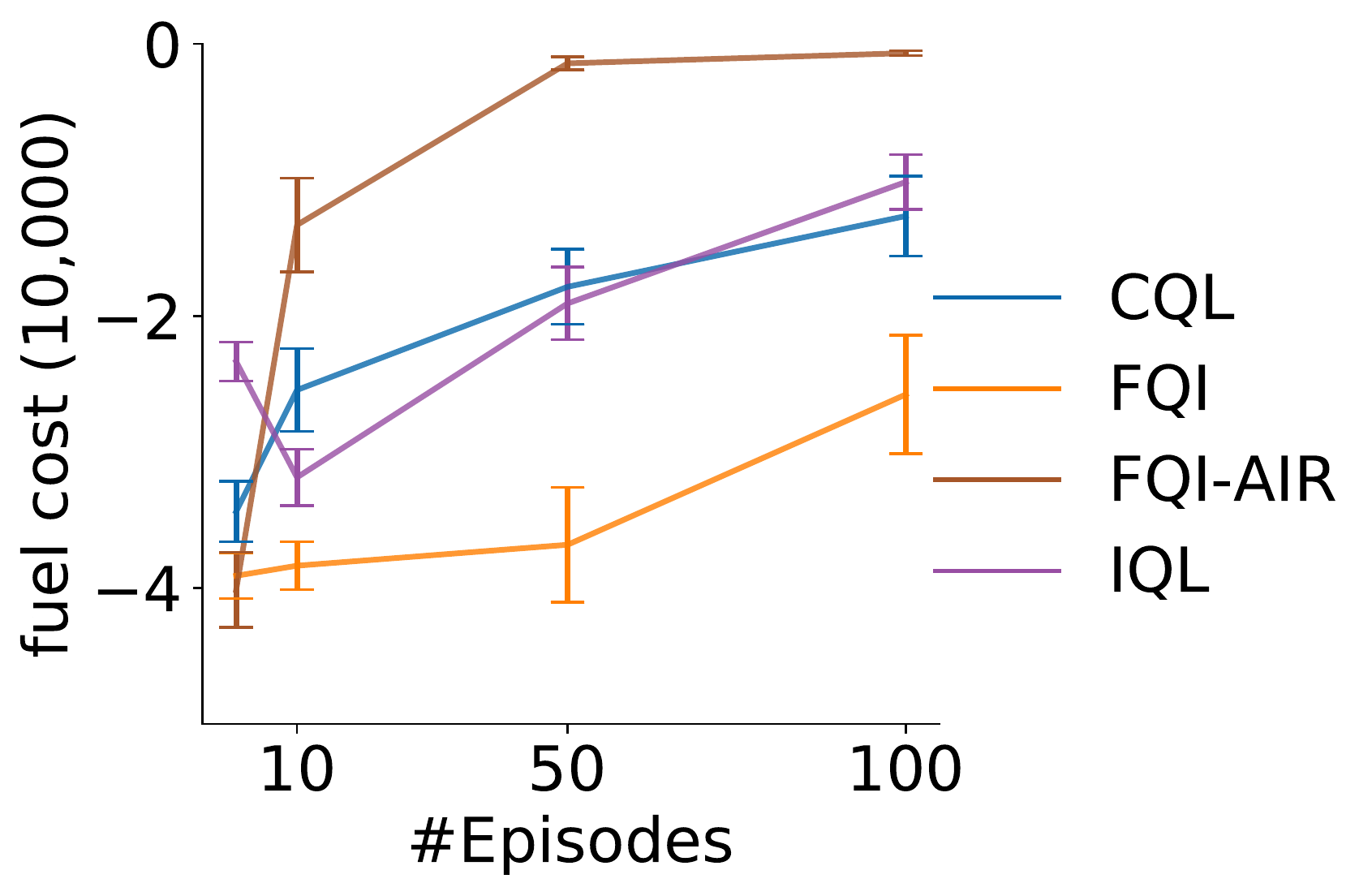}
    }
    \caption{Performance on real world datasets.
    For (a), the numbers represent the total selling price minus the average price. For (b), the numbers represent the fuel cost with a penalty for not maintaining a desired batter level. 
    Results are averaged over 30 runs, shown with one standard error.
    }
    \label{exp:realworld}
\end{figure}

\section{Learning an Endogenous Model in AIR-MDPs}
In the previous experiments, we assume we are given the endogenous models. In this section, we investigate the impact of using an approximate endogenous model learned from offline data. We perform this experiment in the hybrid car environment, which reflects a setting where the endogenous dynamics might in fact not be known and would be useful to estimate from data. 
We use a neural network to approximate the endogenous state and reward model, and run FQI-AIR with the learned endogenous model. 

Let us first reason about when it might be feasible to learn a reasonably accurate endogenous model.
In the worse case, learning an accurate endogenous and reward model would require data coverage for the entire state space. 
However, in many practical scenarios, the endogenous model can be easy to learn and does not require full data coverage. 
For example, in the optimal order execution problem, the endogenous dynamics does not depend on the exogenous variables, as a result, we only need coverage for the endogenous state. 
In the Prius environment, collecting data from just one driving cycle is sufficient to learn a good endogenous model, as we will demonstrate in these experiments.

We first collect a dataset from the hybrid car environment by running a random policy in one of the driving cycles and a deterministic policy for the other driving cycles. This data generation approach mimics a scenario we might see in practice. In the factory, we might have a test system for which it is acceptable to try many different actions (using gas or the battery), and so get a more varied dataset for learning the endogenous model. We would only get this data from one limited course (one driving cycle). The rest of the data would be collected in the wild, where the deployed solution should not be exploring many actions and should largely be deterministic.  

We also test two model-based baselines (MB):
(1) The first baseline has full knowledge of the reward and endogenous models, and learns the exogenous model from offline data without exploiting the AIR property. 
The algorithm is similar to Algorithm \ref{fqi-air} but $\env{s_{h+1}}$ is generated from the learned exogenous model.
(2) The second baseline does not have knowledge of the reward, endogenous models and exogenous models. It learns a full model to from a state-action pair to the next state. 
For these model-based baselines, the model is parameterized by a two-layer neural network and learned by minimizing the $\ell_2$ distance between predicted states and next states recorded in the data. The transitions in these environments are deterministic, so it is appropriate to learn an expectation model. 

In Figure \ref{exp:prius} (a), we perform an ablation study to compare FQI-AIR and MB with the true endogenous model or a learned endogenous model. 
The result shows that 
(1) MB with the true endogenous model performs slightly worse than FQI-AIR with a small data size. 
(2) FQI-AIR with a learned endogenous model perform worse than FQI-AIR, however, it outperforms IQL and MB without the true endogenous model.  
(3) MB with a learned endogenous model performs worse than FQI-AIR with a learned endogenous model. This suggests that it is useful to separate the exogenous state and endogenous state especially when we need to learn an endogenous model.

Next, we test FQI-AIR when learning the endogenous model only from a more limited dataset: a dataset based solely on one cycle. 
We collect a dataset from the hybrid car environment by running a random policy in one of the driving cycles for 500 episodes. 
This reflects a practical scenario that we can just run our vehicle in a closed area and still are able to obtain a good endogenous model for running FQI-AIR.
Figure \ref{exp:prius} (b) shows that FQI-AIR with the learned endogenous model performs well and is close to FQI-AIR with true endogenous model.

\begin{figure}[t]
    \centering
     \subfigure[Comparison to model-based baselines]
    {
        \centering
        \includegraphics[scale=0.3]{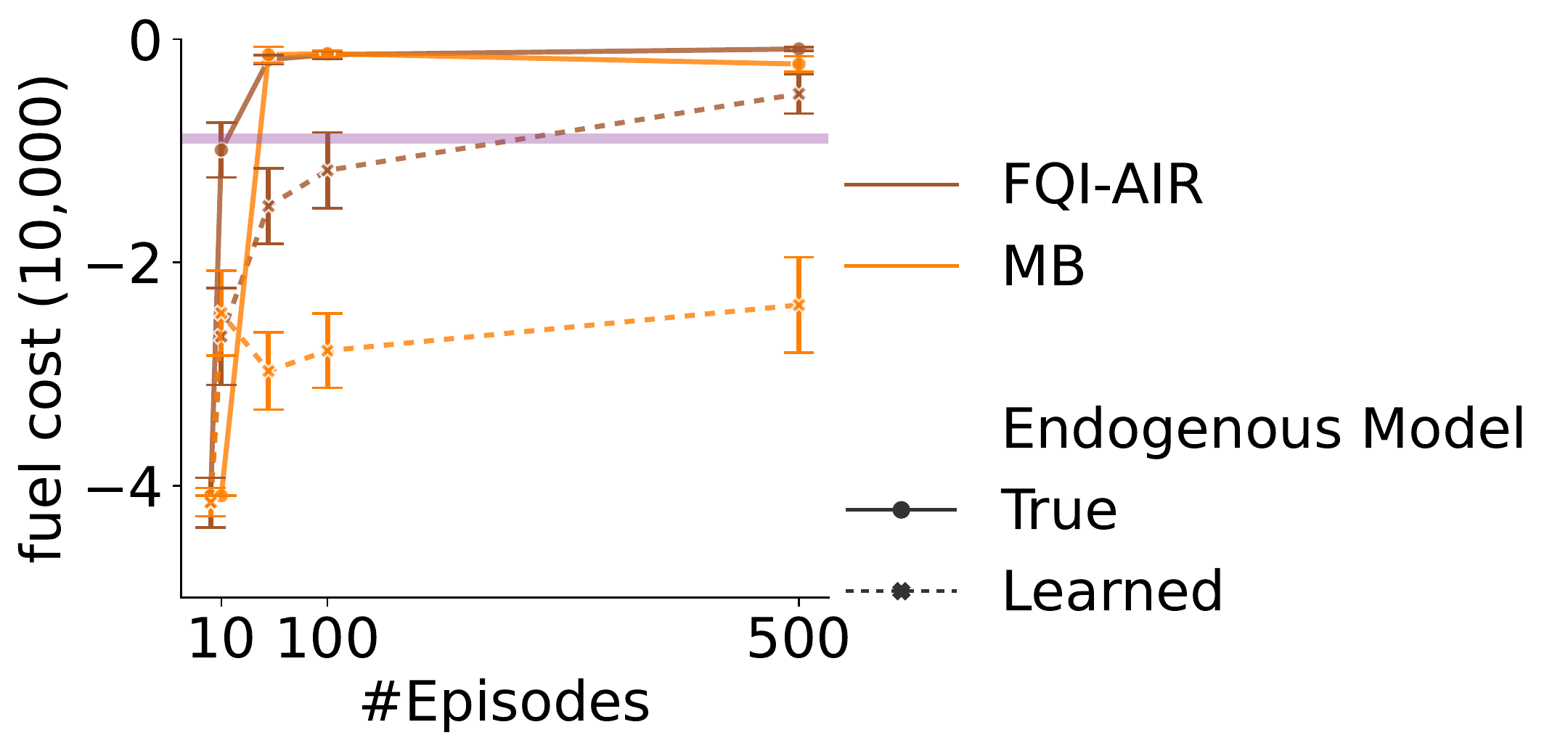}
    }
    \subfigure[Learning the endogenous model with data from one cycle]
    {
        \centering
        \includegraphics[scale=0.3]{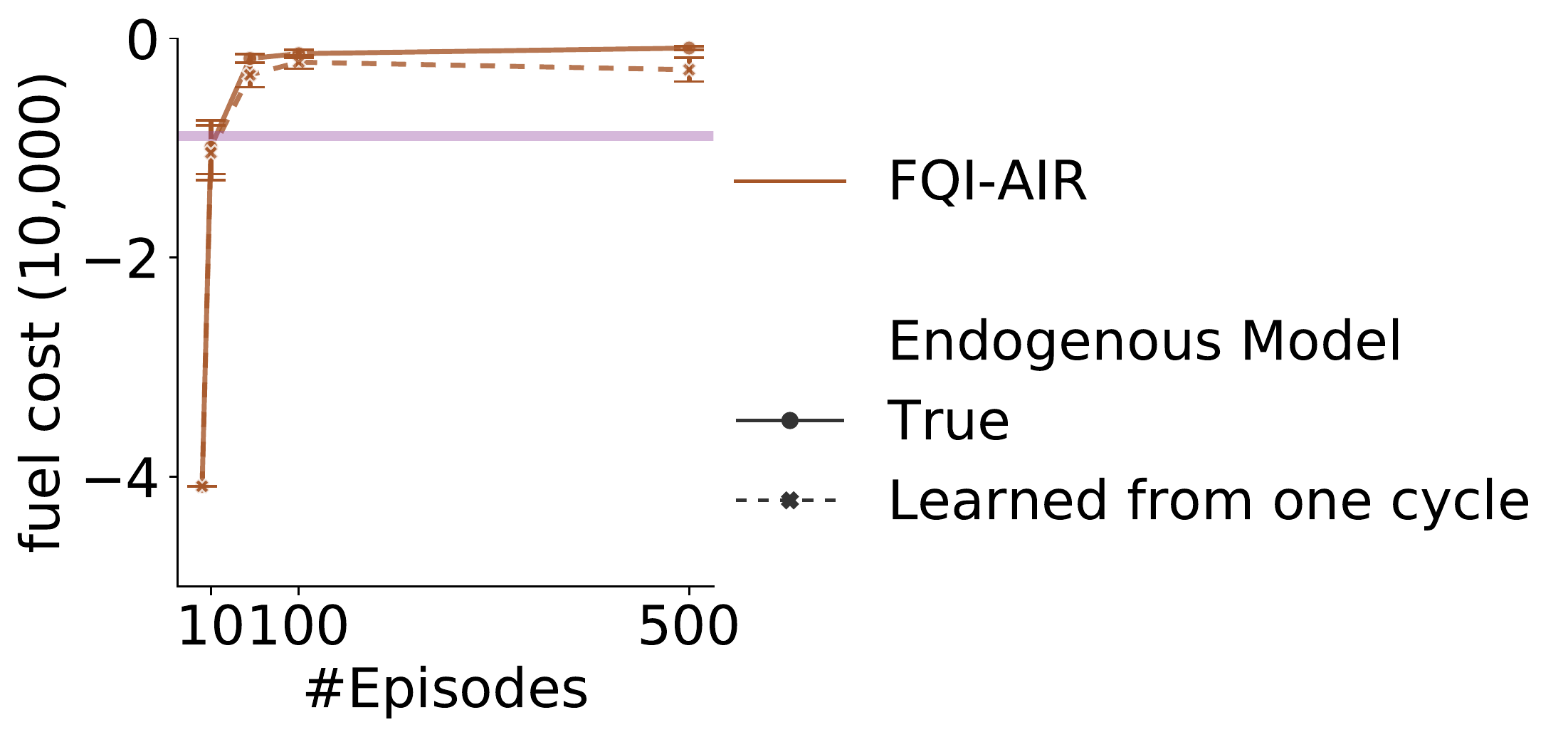}
    }
    \caption{Experiment on the Prius dataset. We include IQL with 500 episodes as a baseline (purple line). Results are averaged over 30 runs, shown with one standard error.
    }
    \label{exp:prius}
\end{figure}

\section{Conclusion}
\label{section:conclusion}
In this paper, we aim to understand whether and when offline RL is feasible for real world problems. It is known that, without extra assumptions on MDPs, offline RL requires exponential numbers of samples to obtain a nearly optimal policy with high probability even if we have a good data distribution. As a result, we must make assumptions on MDPs to make learning feasible (learning with polynomial sample complexity). However, common assumptions can be impractical as discussed in Section \ref{sec:ass}. Therefore, our goal in this paper is to study an MDP property that (1) is realistic for several important real world problems and (2) makes offline RL feasible. 

We introduced an MDP property, which we call Action Impact Regularity (AIR). We developed an algorithm for MDPs satisfying AIR, that (1) has strong theoretical guarantees on the supoptimality, without making assumptions about concentration coefficients or data coverage, (2) provides a simple way to select hyperparameters offline, without introducing extra hyperparameters and (3) is simple to implement and computationally efficient.
We showed empirically that the proposed algorithm significantly outperforms existing offline RL algorithms, across two simulated environments and two real world environments.

\appendix

\section{Experiment Details}
\label{section:details}

In this section, we provide more experimental details.

\paragraph{Algorithm Details}

The original MBS-QI algorithm does not work with negative rewards. The reward in the inventory problem is in the range $[-100,0]$, so we modify the pessimistic value for MBS-QI for the problem:
\begin{align*}
    \Tilde q_h(s,a) \defeq 
    \begin{cases}
        q_h(s,a), &\text{if }\hat\mu(s,a) \geq b\\
        -10000, &\text{if }\hat\mu(s,a) < b
    \end{cases}
\end{align*}
instead of $\Tilde q_h(s,a) \defeq \II\{\hat\mu(s,a) \geq b\} q_h(s,a)$ where $\hat\mu$ is the estimated data distribution. 

\paragraph{Data Collection Policies}
To collect data for the learned policy, we first train a DQN algorithm for 1000 episodes with online interaction with the underlying environment. The DQN parameters (that is, learning rate and optimizer) are chosen based on estimated $J(\pi,M)$. After training, we collect data by running the trained policy. 

\section{Theoretical Analysis}

We first provide definitions and lemmas that would be useful for proving our main theorems. 
Given $q\in\RR^{\cS\times\cA}$, $\beta_h$ be a probability measure on $\cS\times\cA$ and $p>0$, define $\norm{q}_{p,\beta_h}^p = \sum_{s\in\cS, a\in\cA} \beta_h(s,a) |q(s,a)|^p$. Given $q\in\RR^{\cS\times\cA}$, we define the Bellman evaluation operator for a given policy $\pi$:
\begin{align*}
    (\bellman_{\pi}q) (s,a) \defeq r(s,a) + \sum_{s'\in\cS} P(s,a,s') \sum_{a'\in\cA} \pi(a'|s') q(s',a').
\end{align*}

\begin{lemma}
    \label{pdlemma}
    Given a MDP $M$, suppose the sequence of value function $(q_0,\dots,q_{H-1})$ satisfies that $\norm{\bellman q_{h+1} - q_h}_{2,d^\pi_h} \leq \varepsilon$ for all policy $\pi$ with $q_H=0$. Let $\pi$ be the greedy policy with respect to $(q_0,\dots,q_{H-1})$, then we have that
    \begin{align*}
        J(\pi,M) \geq J(\pi_{M}^*,M) - (H+1)H\varepsilon.
    \end{align*}
\end{lemma}
\begin{proof}
    By the performance difference lemma (Lemma 5.2.1 of \citeA{kakade2003sample} or \citeA{chen2019information}), we get
    \begin{align*}
        &J(\pi^*, M_D)- J(\pi,M_D) \\
        &= \EE^{\pi^*}\left[\sum_{h=0}^{H-1} q^*(S_h,A_h) - q^*(S_h,\pi_h(S_h))\right] \\
        &\leq \EE^{\pi^*}\left[\sum_{h=0}^{H-1} q^*(S_h,A_h) - q_h(S_h,A_h) + q_h(S_h,\pi_h(S_h)) - q^*(S_h,\pi_h(S_h))\right] \\
        &\leq \EE^{\pi^*}\left[\sum_{h=0}^{H-1} |q^*(S_h,A_h) - q_h(S_h,A_h)| + |q_h(S_h,\pi_h(S_h)) - q^*(S_h,\pi_h(S_h))|\right] \\
        &\leq \sum_{h=0}^{H-1} \norm{q^* - q_h}_{1,d^{\pi^*}_h\pi^*} + \norm{q^* - q_h}_{1,d^{\pi^*}_h\pi_h} \\
        &\leq \sum_{h=0}^{H-1} \norm{q^* - q_h}_{2,d^{\pi^*}_h\pi^*} + \norm{q^* - q_h}_{2,d^{\pi^*}_h\pi_h}
    \end{align*}
    where $d^\pi_h$ is the state-action distribution at horizon $h$ induced by policy $\pi$. The first inequality follows because $\pi_h$ is greedy with respect to $q_h$.
    We consider a state-action distribution $\beta_0$ that is induced by some policy, then
    \begin{align*}
        \norm{q^* - q_0}_{2,\beta_0} 
        &\leq \norm{\bellman q^* - \bellman q_{1} + \bellman q_{1} - q_0}_{2,\beta_0} \\
        &\leq \norm{\bellman q^* - \bellman q_{1}}_{2,\beta_0} + \norm{\bellman q_1 - q_0}_{2,\beta_0} \\
        &\leq \norm{q^* - q_1}_{2,\beta_{1}} + \varepsilon
    \end{align*}
    where $\beta_1(s',a') = \sum_{s,a} \beta_0(s,a) P(s,a,s')\II\{a'=\arg\max_{a''\in\cA} (q^*(s',a'') - q_1(s',a''))^2\}$ is also induced by some policy. 
    The first inequality follows by the fact that $q^*$ is the fixed point of the operator $\bellman$. 
    We can recursively apply the same process for $\norm{q^*- q_h}_{2,\beta_h}$, $h>0$, and we can get 
    \begin{align*}
        \norm{q^* - q_h}_{2,\beta_h} \leq (H-h) \varepsilon.
    \end{align*}
    Plug in the inequality to the performance difference lemma, we get
    \begin{align*}
        J(\pi^*, M_D)- J(\pi,M_D) 
        &\leq \sum_{h=0}^{H-1} 2 (H-h) \varepsilon  = (H+1)H \varepsilon.
    \end{align*}
\end{proof}

The simulation lemma was first introduced for discounted setting in \cite{kearns2002near}, and here we prove a modified version of the simulation lemma for the finite horizon MDPs. 
\begin{restatable}[Finite Horizon Simulation Lemma]{lemma}{simlemma}
    \label{simlemma}
    Given a policy $\pi$, 
    let $M=(\cS,\cA,P,r,H,\mu)$ and $\hat M=(\cS,\cA,\hat P,r,H,\mu)$ be two finite horizon MDPs and $\pi$ be a policy satisfying that for all $s\in\cS$ and $a\in\cA$, $\norm{P(s,a) - \hat P(s,a)}_{1} \leq \varepsilon$.
    Then,  $|v^\pi_M(s) - v^\pi_{\hat M}(s)| \leq \varepsilon \frac{H^2 \rmax}{2}$ for $s\in\cS_0$.
\end{restatable}
\begin{proof}
    For all $s\in\cS_0,a\in\cA$, 
    \begin{align*}
        & |v^\pi_M(s) - v^\pi_{\hat M}(s)| \\
        &= |r_\pi(s) + \sum_{a\in\cA} \pi(a|s) \sum_{s'\in\cS_1}  P(s,a,s') v^\pi_M(s') - r_\pi(s) - \sum_{a\in\cA} \pi(a|s) \sum_{s'\in\cS_1} \hat P(s,a,s') v^\pi_{\hat M}(s')| \\
        &= |\sum_{a\in\cA} \pi(a|s) \sum_{s'\in\cS_1} P(s,a,s') v^\pi_M(s') - \sum_{a\in\cA} \pi(a|s) \sum_{s'\in\cS_1} \hat P(s,a,s') v^\pi_{\hat M}(s')| \\
        &= |\sum_{a\in\cA} \pi(a|s)\sum_{s'\in\cS_1} P(s,a,s') v^\pi_M(s') - \sum_{a\in\cA} \pi(a|s)\sum_{s'\in\cS_1} \hat P(s,a,s')v^\pi_{M}(s') \\&\ \ \ \ + \sum_{a\in\cA} \pi(a|s)\sum_{s'\in\cS_1} \hat P(s,a,s') v^\pi_{M}(s') - \sum_{a\in\cA} \pi(a|s)\sum_{s'\in\cS_1} \hat P(s,a,s') v^\pi_{\hat M}(s')| \\
        &\leq \sum_{a\in\cA} \pi(a|s) |\sum_{s'\in\cS_1} (P(s,a,s') - \hat P(s,a,s')) v^\pi_{M}(s')| + \max_{s'\in\cS_1} |v^\pi_{M}(s') - v^\pi_{\hat M}(s')| \\
        &\leq \varepsilon (H-1) \rmax + \max_{s'\in\cS_1} |v^\pi_{M}(s') - v^\pi_{\hat M}(s')| \\
        &\leq \varepsilon (H-1) \rmax + \varepsilon_P (H-2) \rmax + \dots + \varepsilon_P 1 \rmax \\
        &\leq \varepsilon \frac{H^2 \rmax}{2} = \varepsilon \frac{H\vmax}{2}
    \end{align*}
    The first equality follows from the Bellman equation. The second and third inequalities follow from that $v^\pi_{\hat M}(s)$ is at most $(H-h)\rmax$ for $s\in\cS_h$.
\end{proof} 

\begin{lemma}
    \label{lemma3}
    \sloppy Let $M=(\cS,\cA,\Penv,\Ppri,r,H)$ be an $\varair$-AIR MDP and $M_b=(\cS,\cA,\TPenv,\TPpri,r,H)$ with $D_{TV}(\Ppri(s,a),\TPpri(s,a))\leq\varp$, then for any policy $\pi$, $$|J(\pi, M) - J(\pi, M_{b})| \leq \vmax H(\varair+\varp).$$
\end{lemma}
\begin{proof}
    Since $M$ is $\varair$-AIR, we have 
    \begin{align*}
        D_{TV}\left(\Penv(\senv,a), \Penv(\senv,a')\right) 
        = \frac{1}{2}\mynorm{\Penv(\senv,a)-\Penv(\senv,a')}_1 
        \leq \varair
    \end{align*}
    Let $e(s,a,\sppri) = \Ppri(s,a,\sppri)-\TPpri(s,a,\sppri)$, we know $D_{TV}\left(\Ppri(s,a), \TPpri(s,a)\right) = \frac{1}{2}\mynorm{e(s,a,\cdot)}_1 = \frac{1}{2}\sum_{\sppri} |e(s,a,\sppri)| \leq \varp$.
    For any $s=(\senv,\spri)\in\Senv\times\Spri$ and $a\in\cA$, we have
    \begin{align*}
        &\mynorm{P(s,a)-\Tilde P(s,a)}_1 \\
        &= \sum_{s'} |P([\senv,\spri],a,s')-\Tilde P([\senv,\spri],a,s')| \\
        &= \sum_{s'} |\Penv(\senv,a,\spenv)\Ppri(s,a,\sppri)-\TPenv(\senv,a,\spenv)\TPpri(s,a,\sppri)| \\
        &= \sum_{s'} |\Penv(\senv,a,\spenv)(\TPpri(s,a,\sppri)+e(s,a,\sppri))-\TPenv(\senv,a,\spenv)\TPpri(s,a,\sppri)| \\
        &= \sum_{s'} |\TPpri(s,a,\sppri)[\Penv(\senv,a,\spenv)-\TPenv(\senv,a,\spenv)]+\Penv(\senv,a,\spenv)e(s,a,\sppri)| \\
        &\leq \sum_{s'}\TPpri(s,a,\sppri)|\Penv(\senv,a,\spenv)-\sum_{a'}\pi'(a'|\senv)\Penv(\senv,a',\spenv)|+ 
        \\&\ \ \ \ \sum_{s'}\Penv(\senv,a,\spenv)|e(s,a,\sppri)| \\
        &\leq \sum_{a'}\pi'(a|\senv) \sum_{\spenv}|\Penv(\senv,a,\spenv)-\Penv(\senv,a',\spenv)|+\sum_{\sppri}|e(s,a,\sppri)| \\
        &\leq 2\varair + 2\varp.
    \end{align*}
    The first inequality follows by writing $\TPenv(\senv,a,\spenv)=\sum_{a'\in\cA}\pi'(a'|\env{s})\Penv(\senv,a',\spenv)$ where $\pi'(a'|\env{s_h}) = \frac{\bbP^{\pi_b}_M(\env{S_h}=\env{s_h},A_h=a')}{\bbP^{\pi_b}_M(\env{S_h}=\env{s_h})}$ when the denominator is nonzero, and otherwise let $\pi'(\cdot|\senv_h)$ be an arbitrary distribution.
    Applying Lemma \ref{simlemma}, we get
    \begin{align*}
        |J(\pi, M) - J(\pi, M_{b})| 
        &= \sum_{s_0} \mu(s_0) |v_{M}^\pi(s_0) - v_{M_b}^\pi(s_0)| 
        \leq \vmax(\varair+\varp) H.
    \end{align*}
\end{proof}

\theoremfqi*
\begin{proof}
    
First fix a horizon $h\in[H]$ and $f'=q_{h+1}\in\cF$.
Define 
\begin{align*}
    R(f) &= \norm{\bellman f' - f}_{2,\tilde \nu_h}^2 \\ 
    &= \sum_{\env{s_h} \in \env{S_h}} \sum_{\private{s_h}\in\private{S_h}} \sum_{a\in\cA} \frac{\tilde \nu_h(\senv_h)}{|\Spri||\cA|} (f(\senv_h,\spri_h,a) - r - \EE[\max_{a'} f'(\env{S_{h+1}},\private{S_{h+1}},a')] )^2
\end{align*}
and 
\begin{align*}
    R_n(f) = \frac{1}{N}\sum_{i=1}^N \sum_{\private{s_h}\in\private{S_h}} \sum_{a\in\cA} \frac{1}{|\Spri||\cA|} (f(s^{(i),\tenv}_h,\private{s_h},a) - r - \max_{a'} f'( s^{(i),\tenv}_{h+1},\private{s_{h+1}},a'))^2 
\end{align*}
where $\private{s_{h+1}} \sim \HPpri(s^{(i),\tenv}_h,\private{s_h},a)$.
Let 
\begin{align*}
    \hat f = \arg\min_{f\in\cF} R_n(f), \tilde f = \arg\min_{f\in\cF} R(f),
\end{align*}
our goal is to bound $R(\hat f)$ with high probability. This is similar to bounding the generalization error in the statistical learning literature. We follow the proof technique of Lemma A.11 in \citeA{agarwal2019reinforcement} to bound the excess risk.

Fix $u=(\private{s_h},a)\in\calU \defeq \Spri\times\cA$, let $x_i^u=(\senv_h,\private{s_h},a)$ and $y_i^u=r(\senv_h,\private{s_h},a)+\max_{a'} f'(\senv_{h+1},\spri_{h+1},a')$ with $(x_i^u,y_i^u)\sim\nu$, $f^*(x_i^u)=\EE[y_i^u|x_i^u]$. Note that our goal is to bound $R(\hat f)=\frac{1}{|\calU|}\sum_{u\in \calU} \EE[(\hat f(x_i^u) - f^*(x_i^u))^2]$.

First note that
\begin{align*}
    \EE[(f(x_i^u) - y_i^u)^2 - (f^*(x_i^u) - y_i^u)^2] = \EE[(f(x_i^u) - f^*(x_i^u))^2]
\end{align*}
and 
\begin{align*}
    \VV^{u,f} &\defeq \VV[(f(x_i^u) - y_i^u)^2 - (f^*(x_i^u) - y_i^u)^2] \\
    & \leq \EE[((f(x_i^u) - y_i^u)^2 - (f^*(x_i^u) - y_i^u)^2)^2] \\
    & \leq 4\vmax^2 \EE[(f(x_i^u) - f^*(x_i^u))^2].
\end{align*}
We can bound the deviation from the mean for all $f\in\cF$ with one-sided Bernstein’s inequality: the following holds with probability $1-\zeta$,
\begin{align*}
    & \EE[(f(x_i^u) - f^*(x_i^u))^2] - \frac{1}{N}\sum_{i=1}^N [(f(x_i^u) - y_i^u)^2 - (f^*(x_i^u) - y_i^u)^2]  \\
    &\leq \sqrt{\frac{2\VV^{u,f}\ln{(|\cF|/\zeta)}}{N}} + \frac{4\vmax^2\ln{(|\cF|/\zeta)}}{3N} \\
    &\leq \sqrt{\frac{8\vmax^2\EE[(f(x_i^u) - f^*(x_i^u))^2]\ln{(|\cF|/\zeta)}}{N}} + \frac{4\vmax^2\ln{(|\cF|/\zeta)}}{3N}.
\end{align*}
We need this holds for all $u\in\calU$. By the union bound, we have for all $u\in\calU$
\begin{align*}
    & \EE[(f(x_i^u) - f^*(x_i^u))^2] - \frac{1}{N}\sum_{i=1}^N [(f(x_i^u) - y_i^u)^2 - (f^*(x_i^u) - y_i^u)^2] \\
    &\leq \sqrt{\frac{8\vmax^2\EE[(f(x_i^u) - f^*(x_i^u))^2]\ln{(|\cF||\calU|/\zeta)}}{N}} + \frac{4\vmax^2\ln{(|\cF||\calU|/\zeta)}}{3N}.
\end{align*}

Since $\hat f$ is the empirical minimizer, we have
\begin{align*}
    &\frac{1}{|\calU|} \sum_{u\in\calU} \frac{1}{N}\sum_{i=1}^N [(\hat f(x_i^u) - y_i^u)^2 - (f^*(x_i^u) - y_i^u)^2] \\
    &\leq \frac{1}{|\calU|} \sum_{u\in\calU} \frac{1}{N}\sum_{i=1}^N [(\tilde f(x_i^u) - y_i^u)^2 - (f^*(x_i^u) - y_i^u)^2].
\end{align*}
Since $\tilde f \in \cF$, we have
\begin{align*}
    &\frac{1}{N}\sum_{i=1}^N [(\tilde f(x_i^u) - y_i^u)^2 - (f^*(x_i^u) - y_i^u)^2] - \EE[(\tilde f(x_i^u) - f^*(x_i^u))^2]  \\
    &\leq \sqrt{\frac{8\vmax^2\EE[(\tilde f(x_i^u) - f^*(x_i^u))^2]\ln{(|\cF||\calU|/\zeta)}}{N}} + \frac{4\vmax^2\ln{(|\cF||\calU|/\zeta)}}{3N}.
\end{align*}
Since $\frac{1}{|\calU|}\sum_{u\in\calU} \EE[(\tilde f(x_i^u) - f^*(x_i^u))^2] \leq \varapx$, we have
\begin{align*}
    &\frac{1}{|\calU|}\sum_{u\in\calU} \frac{1}{N}\sum_{i=1}^N [(\tilde f(x_i^u) - y_i^u)^2 - (f^*(x_i^u) - y_i^u)^2] \\
    &\leq \varapx + \sqrt{\frac{8\vmax^2\varapx\ln{(|\cF||\calU|/\zeta)}}{N}} + \frac{4\vmax^2\ln{(|\cF||\calU|/\zeta)}}{3N} \\
    &\leq \frac{3}{2}\varapx + \frac{16\vmax^2\ln{(|\cF||\calU|/\zeta)}}{3N}.
\end{align*}
The second inequality follows by the fact that $\sqrt{ab} \leq \frac{a+b}{2}$ for $a,b \geq 0$.
Then,
\begin{align*}
    &\frac{1}{|\calU|}\sum_{u\in\calU} \EE[(\hat f(x_i^u) - f^*(x_i^u))^2] \\
    &\leq \sqrt{\frac{8\vmax^2\frac{1}{|\calU|}\sum_{u\in\calU}\EE[(\hat f(x_i) - f^*(x_i))^2]\ln{(|\cF||\calU|/\zeta)}}{N}} + \frac{20\vmax^2\ln{(|\cF||\calU|/\zeta)}}{3N} + \frac{3}{2}\varapx.
\end{align*}
Solving for the quadratic formula, we get
\begin{align*}
    \frac{1}{|\calU|}\sum_{u\in\calU} \EE[(\hat f(x_i^u) - f^*(x_i^u))^2] 
    \leq \frac{36\vmax^2\ln{(|\cF||\calU|/\zeta)}}{N} + 2\varapx.
\end{align*}

Now define $\tilde d^\pi_h(\senv)=\bbP^{\pi}_{M_b}(S^{\tenv}_h=\senv_h)$ and $d^\pi_h(\senv,\spri,a)=\bbP^{\pi}_{M_b}(S^{\tenv}_h=\senv_h,S^{\tpri}_h=\senv_h,A_h=a)$. By the construction of $M_b$, we have 
\begin{align*}
    &\forall \pi,\env{s_h}\in\env{S_h}, \frac{\tilde d^\pi_h(\senv)}{\tilde \nu_h(\senv)} \leq 1 \text{, and }\\
    &\forall \pi,\env{s_h}\in\env{S_h},\private{s_h}\in\private{S_h},a\in\cA,\frac{d^\pi_h(\senv,\spri,a)}{\tilde \nu_h(\senv)/|\Spri||\cA|} \leq |\Spri||\cA|.
\end{align*}
Therefore, for any policy $\pi$, we have with probability $1-\zeta$ 
\begin{align*}
    \norm{\bellman q_{h+1} - q_h}_{2,d^\pi_h} 
    &\leq \sqrt{|\Spri||\cA|} \norm{\bellman q_{h+1} - q_h}_{2,\tilde \nu_h} \\
    &\leq \sqrt{|\Spri||\cA|} (\sqrt{\frac{36\vmax^2\ln{(|\cF||\Spri||\cA|/\zeta)}}{N}+2\varapx}).
\end{align*}
Note that this holds for a fixed $h\in[H]$ and $f'\in\cF$. Use the union bound, we have that, for any policy $\pi$, $h\in[H]$ and $q_{h+1}\in\cF$, with probability $1-\zeta$
\begin{align*}
    \norm{\bellman q_{h+1} - q_h}_{2,d^\pi_h} \leq \sqrt{|\Spri||\cA|} (\sqrt{\frac{72\vmax^2\ln{(H|\cF||\Spri||\cA|/\zeta)}}{N} + 2\varapx}).
\end{align*}

By Lemma \ref{pdlemma}, the output policy $\pi$ satisfies
\begin{align*}
    &J(\pi^*_{M_b},M_b) - J(\pi,M_b) \\
    &\leq (H+1)H\sqrt{ |\Spri||\cA|} (\sqrt{\frac{72\vmax^2\ln{(H|\cF||\Spri||\cA|/\zeta)}}{N} + 2\varapx}) = \varepsilon_1.
\end{align*}

By Lemma \ref{lemma3}, the output policy $\pi$ satisfies
\begin{align*}
    &J(\pi^*_M,M) - J(\pi,M) \\
    &= J(\pi^*_M,M) - J(\pi^*_M,M_b) + J(\pi^*_M,M_b) - J(\pi,M) \\
    &\leq J(\pi^*_M,M) - J(\pi^*_M,M_b) + J(\pi^*_{M_b},M_b) - J(\pi,M) \\
    &\leq J(\pi^*_M,M) - J(\pi^*_M,M_b) + J(\pi,M_b) - J(\pi,M) + \varepsilon_1\\
     &\leq | J(\pi^*_M,M) - J(\pi^*_M,M_b)| + |J(\pi,M_b) - J(\pi,M)| + \varepsilon_1\\
    &\leq 2 \vmax H (\varair+\varepsilon_{P}) + (H+1)H\sqrt{ |\Spri||\cA|} (\sqrt{\frac{72\vmax^2\ln{(H|\cF||\Spri||\cA|/\zeta)}}{N} + 2\varapx}).
\end{align*}
Note that $\sqrt{|\Spri||\cA|}$ comes from the fact that we run FQI for all endogenous state and action uniformly. This term can be viewed as the concentration coefficient in the standard FQI analysis and is unavoidable. However, if we can adapt the weight on the endogenous state and action, we might be able to reduce the dependence.
\end{proof}

\theoremope*
\begin{proof}
    
    We first show that $R(\tau_D^{(i)})\defeq\sum_{t=0}^{H-1} r(s_t^{(i)}, a_t^{(i)})$ for $i\in[N]$ are i.i.d. samples with mean $J(\pi, M_b)$. Let $\bbP_D^\pi$ be the probability measure on trajectories $\tau_D$. Note that the randomness of $\tau_D$ comes from the generation of exogenous variables in by the interaction between $\pi_b$ and $M$, and the generation of actions and endogenous variables by $\pi$ and $\HPpri$.
    Let $\bbP_{M_b}^\pi$ be the probability measure on the trajectories sampled by running $\pi$ on $M_b$. It is sufficient to show that $\bbP_D^\pi=\bbP_{M_b}^\pi$ then $\EE[R(\tau_D)] = \int R(\tau) d\bbP_D^\pi(\tau) = \int R(\tau) d\bbP_{M_b}^\pi(\tau) = J(\pi, M_b)$.
    
    For all trajectories $\tau=(s_0,a_0,\dots,s_{H-1},a_{H-1})\in(\cS\times\cA)^H$, we have 
    \begin{align*}
        \bbP_D^\pi(\tau) 
        &= \bbP_{M}^{\pi_b}(s_0,\senv_1,\dots,\senv_{H-1}) \bbP^\pi_{D}(a_0,\spri_1,\dots,\spri_{H-1},a_{H-1} \mid s_0,\senv_1,\dots,\senv_{H-1}) \\ 
        &= \mu(s_0)\pi(a_0|s_0)\bbP^{\pi_b}_M(\env{s_1}|s_0)\HPpri(s_0,a_0,\spri_1) \pi(a_1|s_1) \dots \\
        &= \mu(s_0) \pi(a_0|s_0) \tilde P(s_0,a_0,s_1) \pi(a_1|s_1) \dots = \bbP_{M_b}^\pi(\tau).
    \end{align*}

    By Hoeffding's inequality, with probability at least $1-\zeta$,
    \begin{align*}
        \left|\hat J(\pi, M) - J(\pi, M_b)\right|
        &= \left|\frac{1}{N} \sum_{i=1}^N R(\tau_D^{(i)}) - \EE[R(\tau_D)]\right|
        \leq \vmax \sqrt{\frac{1}{2N} \ln{\frac{2}{\zeta}}}.
    \end{align*}
    
    Finally by Lemma \ref{lemma3}, the followings hold with probability at least $1-\zeta$:
    \begin{align*}
        |\hat J(\pi, M) - J(\pi, M)| 
        &\leq |\hat J(\pi, M) - J(\pi, M_b)| + |J(\pi, M_b) - J(\pi, M)| \\
        & \leq \vmax H (\epsilon_{AIR} + \epsilon_{P}) + \vmax \sqrt{\frac{\ln{(2/\zeta)}}{2N}}.
    \end{align*}
\end{proof}

\section{Comparing FQI-AIR to Online RL with Trajectory Simulation} \label{appendix:other_air}

\begin{figure}[h]
    \centering
    \subfigure[Optimal order execution]
    {
        \centering
         \includegraphics[width=0.3\textwidth]{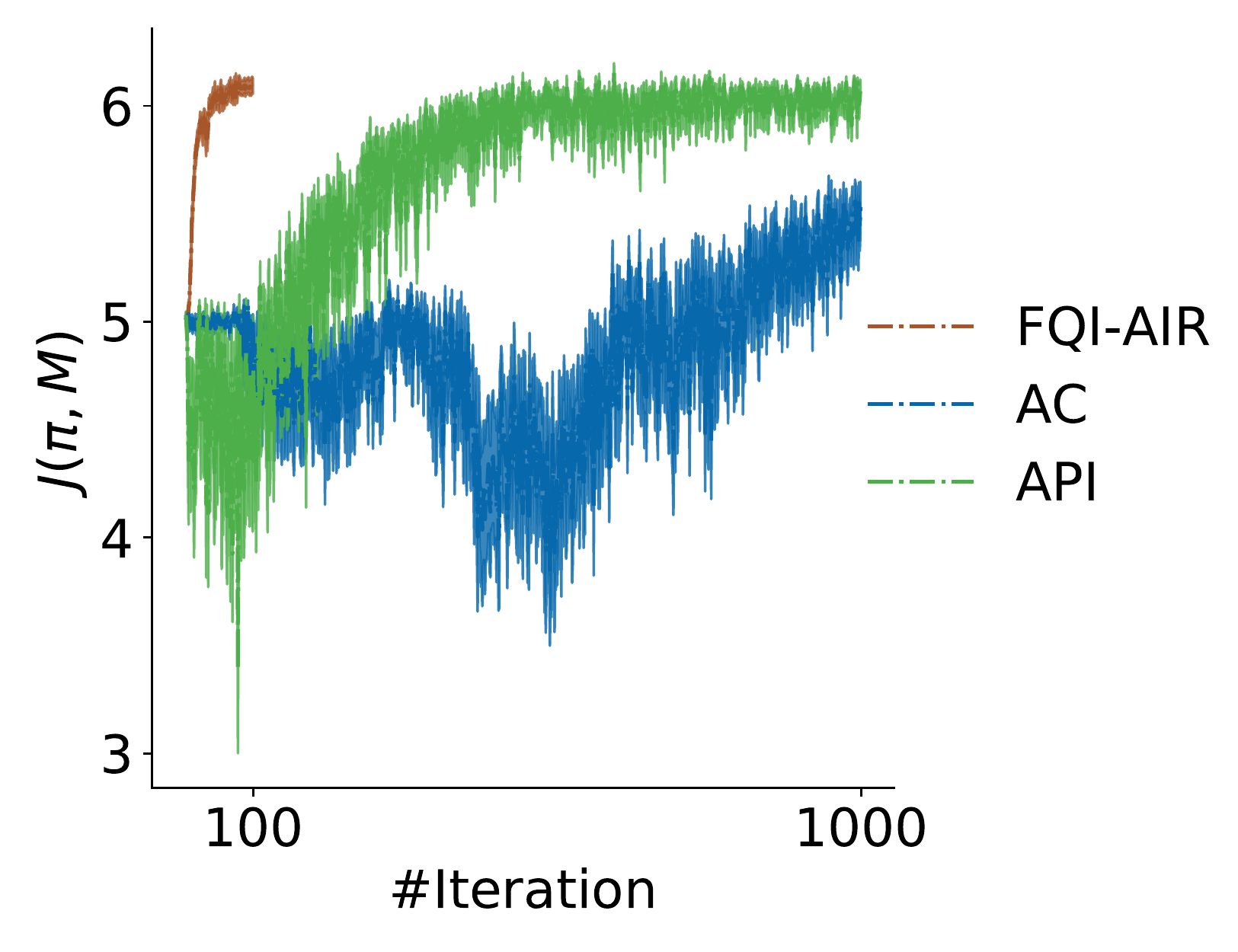}
    }
    \subfigure[Inventory management]
    {
        \centering
        \includegraphics[width=0.3\textwidth]{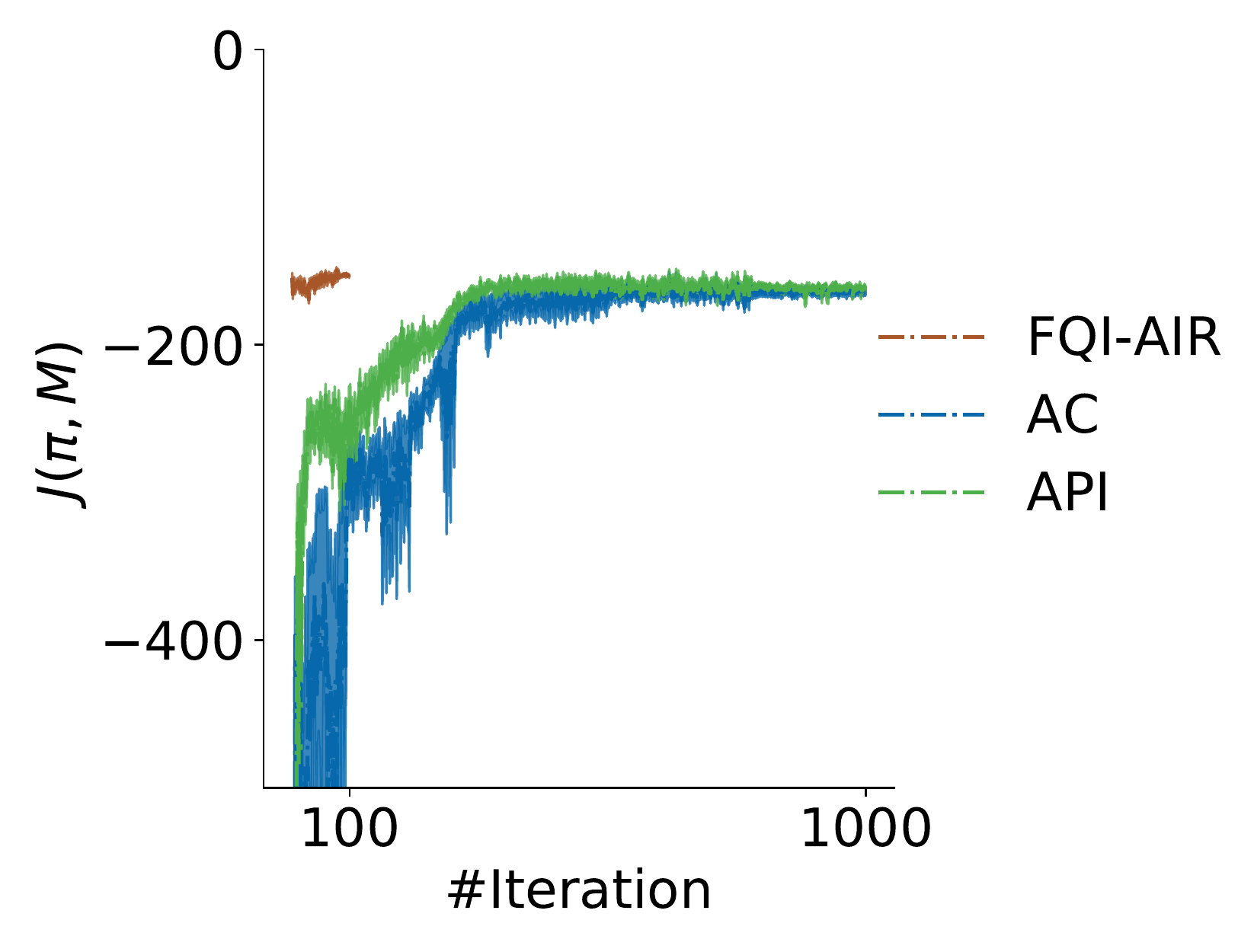}
    }
    \caption{Comparison to simulation-based algorithms. Results are averaged over 30 runs with the shaded region representing plus and minus one standard error.}
    \label{fig:simulation}
\end{figure}

We compare to two online RL algorithms, approximate policy iteration (API) and actor critic (AC) with $\varepsilon$-greedy exploration, with the trajectory simulator discussed in Section \ref{sec:algorithm}. We show the learning curves with random data collection policy and $N=100$. Each iteration contains a sweep over the entire dataset and we evaluate the greedy policy with respect to the Q function after each iteration. The results show that FQI-AIR converges to a nearly optimal solution within a few iterations while the online RL algorithms require much more iterations to find a good policy. AC in the optimal order execution problem does not converge to a stable performance. The final performance of API and AC are also worse than the performance of FQI-AIR.  
These results show that online RL algorithms could be used for AIR MDPs. However, they are less direct and efficient, and they could find a slightly different solution with finite data.

\vskip 0.2in
\bibliography{main}
\bibliographystyle{theapa}

\end{document}